\documentclass[twoside]{article}
\usepackage{ecj,palatino,epsfig,latexsym,natbib}

\usepackage{amsmath}
\usepackage{amssymb}
\usepackage{amsfonts}
\usepackage{mathtools}
\usepackage{mathrsfs}
\usepackage{bm} 
\usepackage{amsthm}

\newtheorem{theorem}{Theorem}
\newtheorem{corollary}{Corollary}

\usepackage[ruled,vlined,linesnumbered]{algorithm2e}
\usepackage{graphicx}
\usepackage[font=small]{caption}
\usepackage{booktabs}
\usepackage{multirow}

\usepackage{enumitem}
\usepackage{xcolor} 

\usepackage{hyperref}

\begin{document}

\ecjHeader{x}{x}{xxx-xxx}{2026}{Advanced Optimizers From Darwinian Evolution}{D. Grimmer}
\title{\bf Direct From Darwin: Deriving Advanced Optimizers From Evolutionary First Principles}  

\author{\name{\bf Daniel Grimmer} \hfill \addr{daniel.grimmer@yale.edu}\\ 
        \addr{Philosophy, Yale University, New Haven, 06511, USA}
}

\maketitle

\begin{abstract}
Evolutionary computation has long promised to deliver both high-performance optimization tools as well as rigorous scientific simulations of Darwinian evolution. However, modern algorithms frequently abandon evolutionary fidelity for physics-inspired heuristics or superficial biological metaphors. This paper derives a suite of advanced gradient-based optimization algorithms directly from evolutionary first principles. We introduce Darwinian Lineage Simulations (DLS) to prove that, in an asexual context, Fisher's and Wright's historically opposed views of evolution are actually formally equivalent; One can partition Fisher's deterministically-evolving total population into Wright's randomly-drifting sub-populations. We prove that proper bookkeeping requires introducing a specific kind of structured noise (the DLS noise relation). Crucially, \textit{any bookkeeping choices} which satisfy this relation will yield a faithful simulation of evolution. Using this vast representational freedom, we prove that a broad family of battle-tested optimization algorithms are already perfectly compatible with evolutionary dynamics. These include: Stochastic Gradient Descent as well as many regularizations/approximations of Newton's method and Natural Gradient Descent. By simply adding DLS noise (i.e., evolutionarily faithful genetic drift), these algorithms become scientifically valid \textit{in silico} simulations of Darwinian evolution. Finally, we demonstrate that even the state-of-the-art Adam optimizer can be brought into evolutionary compliance through a minor mathematical surgery.
\end{abstract}


\section{Introduction}\label{SecIntro}\footnotetext{Code available at \url{https://github.com/danielgrimmer/adam-dls}.}
The field of Evolutionary Computation is founded upon a dual mandate. The first is engineering: to build high-performance optimization algorithms inspired by the mechanisms of Darwinian evolution. The second is scientific: to develop computational platforms that faithfully simulate evolutionary dynamics, thereby enabling controlled in silico experiments on the general principles of Darwinian adaptation. Both purposes appear in the field's founding literature and are rehearsed in its early surveys \citep{FogelReview,SchwefelReview}. Yet in practice, these two aims have not received equal attention, often resulting in algorithms that perform well but lack evolutionary fidelity \citep{MiikkulainenForrest2021}. The engineering strand has historically dominated the field's output, driven by its immediate utility for addressing complex, real-world problems. By contrast, the scientific strand has been comparatively underdeveloped, at least until recently. There is now a growing call for this extensive program of evolutionary inspiration to deliver more scientific value back to the evolutionary biologist.

Fortunately, a rising tide of research in digital evolution has begun to fulfill this scientific mandate.\footnote{See, for instance, its recent applications in studying the possible origins of life \citep{OriginsOfALife}, symbiosis \citep{DigEvoSymbiosis}, and ecology \citep{DigEvoEcology}.} Artificial life platforms such as Tierra \citep{Tierra1991} and Avida \citep{Avida2004} have demonstrated the profound value of moving beyond mathematical approximations to genuinely instantiate evolutionary processes in silico \citep{Pennock2007}. These systems have enabled landmark discoveries, such as empirically demonstrating that complex features can evolve through the accumulation of simpler ``stepping-stone'' traits \citep{Lenski2003}, and that high mutation rates drive populations to favor mutational robustness over peak fitness---a phenomenon known as the ``survival of the flattest'' \citep{Wilke2001}. Contrastingly, the engineering strand is currently facing a ``paradox of success'' \citep{ParadoxOfSuccess}: an explosion of publications introducing supposedly novel algorithms that rely on superficial biological metaphors with little to no actual algorithmic innovation. Critics, notably \cite{Sorensen2015} and the curators of the Evolutionary Computation Bestiary \citep{ECBestiary}, have sharply exposed this metaphor crisis, arguing that the field often reinvents the wheel while obscuring a lack of mathematical rigor. 

To actively avoid superficial metaphors, the optimization framework derived in this paper (\textit{Darwinian Lineage Simulations}, DLS) is firmly grounded on evolutionary first principles. Moreover, the present work deftly avoids contributing to the overabundance of false algorithmic novelty; remarkably, beyond deriving the DLS noise relation, we do not claim to invent any new optimization algorithms. Rather, we demonstrate that a broad family of battle-tested optimization algorithms are already in compliance with our evolutionarily derived update rules. These include: Stochastic Gradient Descent as well as many regularizations/approximations of Newton's method and Natural Gradient Descent (see Sec.~\ref{SecDLSNoiseRel} for details). Once equipped with a biologically faithful form of genetic drift (governed by the DLS noise relation), all of these algorithms are immediately revealed to be scientifically valid in silico simulations of Darwinian evolution. A notable outlier, however, is the state-of-the-art Adam optimizer. In order to be brought into strict evolutionary compliance, Adam's momentum term requires some minor but principled mathematical surgery. We perform this surgery in Sec.~\ref{SecRecoverAdam}.

At this point, one may be wondering: How does evolution produce gradient-based algorithms at all? Isn't Darwinian adaptation supposed to be the paradigm case of a gradient-free optimization process? This persistent association of evolutionary methods with gradient-free optimization stems from the relative dominance of the engineering strand over the scientific strand within evolutionary computation. For the most part, it is only when practitioners cannot apply their preferred tools (gradient-based methods) that they turn to Darwinian evolution as a source of gradient-free heuristics. This association, however, only reflects the history and motivations of a particular computational tradition and is, in fact, contrary to the mathematics of evolution itself. One might object that our biology is fundamentally discrete at the genomic level. Two replies are available. First, evolution is substrate-independent: the Darwinian framework applies wherever one finds inheritance, variation, and differential reproduction (or expansion, or persistence). Alien systems (e.g., neural networks) with continuous genotype spaces are able to perfectly instantiate evolutionary mechanisms \citep{Pennock2007} and can thereby serve as good analogical models for simulating our own biological evolution. 

Second, and more decisively, even granting a discrete substrate at the molecular level, the mathematics of quantitative genetics has always been continuous and gradient-based. \citeauthor{Fisher1930}'s (\citeyear{Fisher1930}) fundamental theorem of natural selection is an ascent theorem; the \cite{Price1970} equation, Eq.~\eqref{EqFisherKeyInsight}, can be seen as a product of population variance and a selection gradient;\footnote{The Price equation is also central to \citeauthor{Grafen2002}'s (\citeyear{Grafen2002}) Formal Darwinism project, which uses it to formally link the dynamics of evolution with mathematical optimization. In particular, he uses it as an ``equation of motion'' to prove that natural selection leads organisms to act as if they are solving an optimization program to maximize their biological fitness.} \citeauthor{Lande1976}'s (\citeyear{Lande1976,Lande1979}) equation, $\Delta\bar{z} = G \nabla \ln \bar{W}$, is literally gradient ascent.\footnote{As \cite{Otwinowski2020} show the pre-conditioner, $G$, in Lande's equation can be given an information geometric interpretation as a natural gradient. Moreover, they observe its similarity to Newton's method in the limit of strong selection.} \cite{Barton2013} have rightly identified that quantitative genetics can act as a mathematical bridge between biology and computation. Indeed, within evolutionary computation, Natural Evolution Strategies (NES) have historically leveraged information geometry to follow the natural gradient of expected fitness. The strict association of evolution with gradient-free methods is therefore an artifact of a particular computational tradition, not a reflection of evolution's own mathematical structure. The concept of a population stochastically climbing a fitness landscape \citep{Wright1931,Wright1932} is not something which has to be artificially imported into Darwinian evolution from optimization theory. Fitness gradients are what selection implicitly and automatically computes. 

Discussions of gradient-based evolution are not absent in the modern machine learning literature. Thus far, however, most (but not all) attempts to formalize this connection have relied on importing notions from statistical physics rather than actual evolutionary theory. As a point of contrast with our evolutionarily faithful approach, let us briefly consider the work of \cite{Whitelam2021} which shows that stochastic mutation of a neural network's weights approximates gradient descent in the presence of Gaussian noise. From an engineering perspective, this is a mathematically clean and useful result: it explains why neuroevolution performs comparably to backpropagation in empirical tests. While mathematically sound, this result is scientifically inert for the evolutionary biologist. Random mutations, $\xi\sim\mathcal{N}(0,\mu^2I)$, are proposed and then accepted ``with the Metropolis probability $\min(1,\exp(-\beta\Delta U))$, a common choice in the physics literature'' \citep[pg.1]{Whitelam2021}.\footnote{Note: We have adapted \citeauthor{Whitelam2021}'s (\citeyear{Whitelam2021}) notation to ours by relabeling their mutation rate from $\sigma$ to $\mu$ and the mutation itself from $\epsilon$ to $\xi$.} Here $\beta$ is the ``reciprocal evolutionary temperature''. These are concepts imported from statistical physics, not from population genetics. By contrast, the Darwinian Lineage Simulations described in this paper are \textit{quite literally} simulations of Darwinian evolution. It should be stressed that the results of this paper would still be significant even if our evolutionarily faithful approach were to yield the exact same update equation as some other physics-based approach; as a matter of scientific hygiene, no evolutionary biologist should use the concept of thermal noise to model genetic drift, and now they don't have to.

A more evolutionarily faithful bridge between gradient descent and evolutionary dynamics has recently been established by \cite{Kucharavy2023}. By formalizing the Gillespie-Orr Mutational Landscapes model from population genetics into a class of evolutionary algorithms (GO-EA), they demonstrate a strict limit equivalence with Stochastic Gradient Descent (SGD). Crucially, this approach successfully returns scientific value to evolutionary theory. For example, it leverages SGD dynamics to reframe ``flat minima'' in machine learning loss landscapes as analogous to the above-discussed ``survival of the flattest'' phenomena. However, while the GO-EA framework successfully grounds basic SGD in strict population genetics, it remains limited to this specific algorithm. Concretely, because the GO-EA algorithmic framework relies on random sampling and a greedy search to approximate the gradient, its updates remain effectively isotropic. By contrast, it follows directly from the Price equation, Eq.~\eqref{EqFisherKeyInsight}, that the population's variance naturally acts as an anisotropic pre-conditioner. Because their GO-EA implementation does not explicitly track or harness an evolving covariance matrix, it is structurally unable to produce a wider family of pre-conditioned gradient algorithms (such as Adam, Newton methods, or Natural Gradient Descent). To find these, one must look to broader theoretical unifications, such as the framework proposed by \cite{Frank2025}.

Indeed, as we shall discuss in Sec.~\ref{SecGenUpdate}, following \citeauthor{Frank2025}'s (\citeyear{Frank2025}) methodology exposes the seeds of gradient-based optimization inside of the Price equation (a foundational, albeit tautological, description of evolutionary change). Building on this insight, Frank derives a universal Force-Metric-Bias (FMB) law, $\Delta\theta = Mf + b + \xi$, demonstrating that everything from stochastic gradient descent and Newton's method to Bayesian updating and natural selection shares the same underlying mathematical skeleton; they differ only in how they specify the metric $M$, the force $f$, the bias $b$, and the noise, $\xi$. 

While the FMB framework provides a genuine and powerful unification, its sheer generality reveals its primary limitation: it is taxonomic rather than generative. Given one's favorite optimization algorithm, the FMB law can effectively categorize its parts (i.e., its particular choices of $M$, $f$, $b$, and $\xi$) and identify their evolutionary analogues. After completing this formal decomposition, however, the FMB law does not help us determine which evolutionary mechanisms might implement this algorithm (e.g., what fitness landscape, what mode of reproduction, what model of child mortality, etc.). Moreover, the FMB decomposition cannot tell us whether or not the optimizer under consideration can be interpreted as a faithful simulation of evolutionary dynamics (e.g., Adam cannot without surgery, see Sec.~\ref{SecRecoverAdam}). This gap is most glaring in \citeauthor{Frank2025}'s (\citeyear[pg.4, pg.16]{Frank2025}) treatment of the stochastic residual term, $\xi$, which is added in merely ``to complete the FMB law'' and ``enhance exploration'' rather than being derived from any actual evolutionary considerations of genetic drift. 

This is precisely where the Darwinian Lineage Simulation (DLS) framework intervenes. We shall derive a suite of advanced gradient-based optimization algorithms from evolutionary first principles. The evolutionary fidelity of these algorithms is grounded in their relation to \citeauthor{Fisher1930}'s (\citeyear{Fisher1930}) and \citeauthor{Wright1931}'s (\citeyear{Wright1931,Wright1932}) bitterly opposed views of evolutionary dynamics.\footnote{See \cite{Provine1986} for historical context on the Fisher-Wright debate.} While previous unifications of these views have historically relied on continuous diffusion approximations,\footnote{\cite{Kimura1964} famously provided a rigorous mathematical reconciliation of Fisher's deterministic selection and Wright's random drift by taking the continuous limit of discrete Wright-Fisher binomial sampling to derive a diffusion (Fokker-Planck) equation. By contrast, \cite{Frank2013} attempts to resolve the Fisher-Wright debate by arguing that their theories are ``incommensurable'' because Fisher was formulating abstract invariance laws while Wright was building detailed dynamical models of evolutionary change.} we show that the tension between them can be resolved into a \textit{formal equivalence} in the context of asexual reproduction. Concretely, the Darwinian Lineage Simulation framework will unite Wright's emphasis on sub-populations undergoing random genetic drift, $\xi_g\sim\mathcal{N}(0,W_g)$, with Fisher's key insight that deterministic mass selection acts upon the population's current variance, $V_g$, here understood as a pre-conditioner/learning rate. As Sec.~\ref{SecLineageVsPopulation} will discuss in detail, our method of harmonizing these two perspectives involves some careful bookkeeping regarding how the total population is carved up into small sub-populations (or, over time, lineages). As we shall prove in Sec.~\ref{SecDownSampling}, the equivalence between Fisher and Wright only holds if one enforces the DLS noise relation, $W_g=\mu^2 I- (V_{g+1} - V_g)$, which strictly couples the rate of genetic drift ($W_g$) to the mutation rate ($\mu$) and changes in the tracked population's variance ($V_{g+1} - V_g$).\footnote{Technically, this is the leading order form of the DLS noise relation. Its full form can be seen in Sec.~\ref{SecDLSNoiseRel}.} This relation is not a modeling choice; genetic drift must look exactly like this if one's bookkeeping is to be evolutionarily faithful.

Apart from enforcing the DLS noise relation, however, one's bookkeeping choices (i.e., regarding how to carve the total population up into small sub-populations) are \textit{completely free} and do not affect the simulation's evolutionary fidelity. This representational freedom stems from the fact that we are working in an asexual context; no matter how we define the sub-populations there will never be any gene flow between them. It is precisely this flexibility which allows the Darwinian Lineage Simulation framework to so effectively encompass both classical and state-of-the-art gradient-based optimization algorithms. It should be stressed that the primary contribution of this work is therefore interpretive and foundational rather than algorithmic novelty. Besides our specification of the DLS noise relation, all of the algorithms discussed in this paper are already well-known. What is new is the revelation that many of these methods already admit evolutionarily faithful derivations. Simply by adopting our first-principles model of genetic drift, they become scientifically valid simulations of Darwinian evolution \textit{in silico}. Finally, for methods that deviate from these biological first principles (most notably the Adam optimizer) we expose exactly what minor but principled mathematical surgery is required to bring them back into strict evolutionary compliance.

The remainder of this paper is organized as follows. Sec.~\ref{SecDarwinianSGA} will develop the DLS framework from evolutionary first principles by proving the asexual Fisher-Wright equivalence. Sec.~\ref{SecDLSNoiseRel} will then show that a wide range of advanced gradient-based optimization algorithms fit the DLS form once they are equipped with an evolutionarily faithful model of genetic drift, namely, the DLS noise relation. Finally, Sec.~\ref{SecRecoverAdam} will perform some minor surgery on Adam's momentum term to bring it into evolutionary compliance.

\section{Darwinian Lineage Simulations and the Asexual Fisher–Wright Equivalence}\label{SecDarwinianSGA}

This section will present Darwinian Lineage Simulations (DLS), an evolutionarily faithful framework for simulating the evolution of an asexual population with continuous genotypes, $\phi \in \mathbb{R}^N$. Our first principles derivation of the DLS framework proceeds directly from \citeauthor{Fisher1930}'s (\citeyear{Fisher1930}) Large Population Size Theory adapted to an asexual context. This theory envisions evolution as a deterministic process of mass selection operating on the (additive) genetic variance within a large (well-mixed) population.\footnote{Importantly, in an asexual context, the entire genome is inherited more-or-less intact; it is subject to mutation \textit{but not} sexual recombination. This means that complex combinations of genes can be reliably passed from one generation to the next. Hence, the distinction between additive genetic variance and total genetic variance collapses \citep{VarDecomp2013}. One of the key points of conflict between Fisher and Wright is thus resolved for asexual reproduction: Fisherian selection can now act upon co-adapted gene complexes (which Wright emphasized) as if they were simple additive units. Said differently, Fisher’s selectionist logic can now play out on Wright’s rugged fitness landscape. In sum: Fisher's restriction of selection to \textit{additive} genetic variance is automatically relaxed in our asexual context. So too is his restriction to populations which are \textit{well-mixed} (panmictic).\label{FnAdditiveVariance}} While random mutations ultimately help support the population's variance, they average out at scale leading to little or no genetic drift \citep{Provine1986}. In line with this Fisherian dynamics, Sec.~\ref{SecGenUpdate} will provide a deterministic update rule which carries the total population's genotype distribution from one generation, $P_g(\phi)$, to the next, $P_{g+1}(\phi)$, via an ever-changing multiplicative fitness function and a blind Gaussian mutation profile.

Ironically, in order to more efficiently carry out this Fisherian dynamics we shall prove that it is \textit{formally equivalent} to its historical rival: \citeauthor{Wright1931}'s (\citeyear{Wright1931,Wright1932}) Shifting Balance Theory (once again, adapted to an asexual context). As a senior animal husbandman at the U.S. Department of Agriculture, Wright claimed that the best results in cattle breeding came not from population-level selection across one massive intermingled herd, but rather from breeding many small local herds in parallel and then selectively exporting the superior bulls to intermix with the rest.\footnote{As \cite{CoyneBartonTurelli1997} note, however, Wright cannot substantiate this claim without comparing his results with a control group undergoing deterministic mass selection. (There was no such control group.)} He hypothesized that evolution could proceed through a functionally similar mechanism. In brief, his theory outlines three phases. Phase I: Small, genetically isolated sub-populations (which he calls demes) undergo random genetic drift across the fitness landscape, potentially by passing through maladaptive valleys. Phase II: Eventually, some of these demes may approach new fitness peaks at which point selection effects will pull them up to their respective summits. Phase III: The demes that reach the highest peaks become disproportionately represented in the total population, specifically by producing a surplus of individuals that differentially migrate into less-fit demes. This is called interdemic selection.

One of the strongest points of contrast between Fisher and Wright is about whether or not random genetic drift plays an important role in evolution.
While \cite{Fisher1950} acknowledged that evolution is shaped by random environmental fluctuations, they insisted that selection itself was a continuous, predictable, and directed force; if the environment creates a stable fitness gradient, any large population will deterministically climb it. \cite{Wright1931,Wright1932} argued, however, that strict determinism was an evolutionary trap; once a population reaches a local peak (or encounters a perfectly flat plateau) the selection pressure would stop indefinitely.\footnote{Wright is incorrect about this. See Sec.~\ref{SecGenUpdate} for a demonstration of how Fisherian dynamics can solve a flat fitness maze without genetic drift.} The need to cross such plateaus and valleys provides some motivation for the key feature of Wright's Phase I: random genetic drift. The desired random noise can only come from drawing a random sample or sub-population from a larger distribution, hence Wright's focus on dividing the total population up into smaller sub-populations. (The puzzle of how Wright's randomly evolving sub-populations can be formally equivalent to Fisher's deterministically evolving total population will be handled in due course, see Sec.~\ref{SecDownSampling}.)

As \cite{CoyneBartonTurelli1997} note, another of Wright's motivations for introducing this subdivided demographic structure was to secure one of the evolutionary advantages of asexual reproduction while operating within a sexual context. Here is the problem: mating between two individuals with and without a certain well-adapted gene complex will invariably break that gene complex. In order for such beneficial gene complexes to evolve (Wright argues), there must be reproductive barriers in place to protect them. This solution, however, raises several further issues \citep{CoyneBartonTurelli1997}. Firstly, Wright's theory depends upon a demographic structure that many biologists argue is rarely realized in nature, i.e., the total population being divided into small sub-populations with a strongly restricted gene flow between them.\footnote{This claim, however, has been contested in metapopulation contexts \citep{WadeGoodnight1998}.} Moreover, it has been argued that his migration mechanism is ad hoc and ecologically implausible: increases in fitness do not reliably lead to increases in emigration. Lastly, the interdemic migration which Wright suggests happens in Phase III, would face the exact same sexual recombination problem that he is trying to avoid. Namely, any highly adapted deme migrating into a less-fit area will interbreed with the locals, breaking apart the precise, co-adapted gene complexes that made it successful in the first place.

All of the above-discussed issues evaporate if one moves Wright's picture into an asexual context. In particular, the fatal flaw in Phase III is eliminated: there is no longer any sexual recombination to break down well-adapted gene complexes. In an asexual context, the entire genome is perfectly heritable, acting as a single, indivisible allele. Moreover, in an asexual context the ``migration'' aspect of Phase III is replaced by the well-documented biological phenomenon of clonal interference \citep{ClonalInterference}. Because asexual lineages cannot interbreed to share beneficial mutations, they must compete directly against one another for dominance in the total population \textit{by sheer demographic weight}. We don't have to posit that the fit demes send out migrants; they just compete straightforwardly via differential survival and reproduction. In our framework, interdemic selection is seamlessly executed via a fitness-weighted reassembly of Wright's demes (see Sec.~\ref{SecLineageVsPopulation}). As we shall prove by the end of this section, this mathematical reassembly of Wright's randomly drifting demes flawlessly reconstructs the deterministic Fisherian mass selection of the total population. The magic formula which unites these opposing perspectives is the DLS noise relation.

Having repaired Phase III, we can now address the implausibility of Wright's demographic structure. Recall \citeauthor{CoyneBartonTurelli1997}'s (\citeyear{CoyneBartonTurelli1997}) complaint about Wright's reliance on small, genetically-isolated sub-populations which are arguably ecologically unrealistic in nature. In Wright's motivating example (cattle breeding), physical barriers were needed between the demes to restrict gene flow (and hence to protect the fragile gene complexes). Because of their sexual context, Wright's demes must be \textit{demographically substantial} (i.e., physically isolated groups). In our asexual context, however, lineages cannot interbreed. As such we have automatically restricted gene flow without needing to invoke physical geographic isolation; the mode of reproduction itself provides an absolute barrier. This marks an important ontological difference between the original version of Wright's Shifting Balance Theory and our asexual rendition of it. For us, Wright's demes are merely an analytic bookkeeping device, carrying no special demographic ontology. To mark this distinction, we shall leave the term ``demes'' to Wright and proceed by calling them ``sub-populations''. \textit{Importantly, no matter how we divide the total population into these Wright-like sub-populations, we can always reassemble them to recover the Fisherian dynamics of the total population.} This representational freedom (which we will formally introduce in Sec.~\ref{SecLineageVsPopulation}) is extremely handy and will later underwrite our ability to derive a wide variety of advanced optimization algorithms from evolutionary first principles.

Sec.~\ref{SecGaussianUpdateRules} then uses a particular sub-population decomposition (i.e., a particular bookkeeping choice) to derive a computationally efficient method for tracking the genotype distribution of a particular kind of Darwinian lineage. The generational update rule from Sec.~\ref{SecGenUpdate} applies to these sub-populations carrying their genotype distribution from one generation, $p_g(\phi)$, to the next, $p_{g+1}(\phi)$, via an ever-changing multiplicative fitness function and a blind Gaussian mutation profile. Theorem~\ref{ThmGaussianLineageProp} proves that if $p_{g}(\phi)$ is Gaussian with sufficiently small variance, then $p_{g+1}(\phi)$ is also Gaussian to a good approximation. Hence, so long as the lineage's variance remains sufficiently small, one only needs to track the first and second moments of its genotype distribution. As we shall see, the resulting generational update rule for these \textit{Darwinian Lineage Simulations} (DLS) is formally equivalent to a certain kind of pre-conditioned Stochastic Gradient Ascent (SGA) up the log-fitness landscape. By repeatedly applying this DLS update rule, we can efficiently track a lineage's genotype distribution across several generations.

Eventually, however, the resulting genotype distribution becomes too unwieldy to track efficiently. Concretely, its variance eventually becomes too large to justify the small-variance assumption that underwrites the DLS update rule. Fortunately, at this point (or, indeed, at any prior point) we can simplify the simulation by choosing to track only a sub-population of this ever-growing lineage. Sec.~\ref{SecDownSampling} formalizes this down-sampling procedure and explains how it introduces genetic drift into our evolutionary simulation. Moreover, it provides an additional variance-framing freedom at each application of the DLS update rule. Importantly, when properly balanced with the DLS noise relation, our control over the variance of the tracked population at each step has \textit{no effect whatsoever} on the total population's Fisherian dynamics; it is always recoverable via an ensemble of Darwinian Lineage Simulations. Crucially, this down-sampling freedom allows us to maintain the small-variance condition and continue tracking the lineage's genotype distribution with the computationally efficient DLS update rule, see Theorem~\ref{ThmDLSUpdatePackage}. 

Moreover, our ability to arbitrarily shape the sub-population's variance (balanced by the DLS noise relation) grants the DLS framework a significant amount of flexibility. In Sec.~\ref{SecLitReview} we will prove that many state-of-the-art gradient-based optimization algorithms become certified simulations of Darwinian evolution, simply by coupling them with the DLS noise relation. This opens the door for us to perform faithful simulations of evolutionary dynamics using the extremely efficient tools of gradient-based optimization.

\subsection{Fisherian Selection and Stochastic Generational Dynamics}\label{SecGenUpdate}
Let us begin by defining a generational update rule that captures how mutations and selection pressures shape the total population's genotype distribution, $P_{g}(\phi)$, from one generation to the next, $P_{g+1}(\phi)$. Unlike many Artificial Life platforms (e.g., Tierra or Avida) where fitness is implicitly evaluated by digital organisms competing for shared computational resources, our framework relies on an explicit, user-defined multiplicative fitness function, $\mathcal{F}(\phi)$. This function aggregates the survival probability and reproductive output of an individual possessing genotype, $\phi$. Crucially, we assume that selection is frequency-independent; an organism's fitness depends only on its genotype and the environment, not on the frequency of competing strategies in the population. For example, \cite{Grimmer2026} has recently shown that the MAML++ meta-objective is a biologically plausible model of childhood mortality where every task failure results in a small chance of instant death. We will soon add stochastic elements to $\mathcal{F}(\phi)$ to capture both intra- and inter-generational variability (denoted $\eta$ and $\lambda$). For now, however, let us proceed abstractly with some generic fitness function. 

If no mutations arose between generations, then the genotype distribution of the next generation, $P_{g+1}(\phi)$, would simply be proportional to $\mathcal{F}(\phi)P_{g}(\phi)$. Importantly, mutations do occur. Whenever an individual reproduces asexually, its continuous genotype, $\phi$, is randomly modified as it is being passed on to each offspring. We model this species' mutation profile with a fixed, isotropic Gaussian distribution, $\mathcal{N}(\phi;0,\mu^2 \, I)$, centered at zero with a constant mutation rate $\mu$.\footnote{We assume a Gaussian mutation profile in line with Fisher's Geometric Model, though we acknowledge that heavy-tailed distributions (e.g., Cauchy or Lévy flights) are increasingly studied for modeling saltational biological searches.} This represents strictly \textit{blind} mutation.\footnote{As we shall discuss in Sec.~\ref{SecLitReview}, this is a critical point of biological fidelity: unlike engineering heuristics such as the Covariance Matrix Adaptation Evolution Strategy (CMA-ES) which actively adapt the mutation profile to favor historical directions of success, true Darwinian mutation does not know anything about the fitness landscape. Further comparison of DLS with CMA-ES is, unfortunately, beyond the scope of this paper.} Combining frequency-independent selection and blind mutation yields the following generational update rule:
\begin{align}\label{EqNonStochGenUpdateRule}
\nonumber
&\text{Non-Stochastic Generational Update Rule:}\\
&P_{g+1}(\phi) \propto \left(\mathcal{F}(\phi) \ P_{g}(\phi)\right) \, \ast \, \mathcal{N}(\phi;0,\mu^2 \, I).
\end{align}
The symbol $\propto$ denotes proportionality (i.e., the right-hand side is unnormalized), and $\ast$ denotes convolution. 

This update equation is Fisherian in the sense that it implements mass selection at a population-level in a deterministic way. To quantitatively characterize this dynamics, we can use the \cite{Price1970} equation to track how the population's mean genotype, $\phi_{g}=\langle\phi\rangle_{g}$, changes from one generation to the next. Given that our mutation kernel has zero mean, the change in $\phi_{g}$ is exactly:
\begin{align}\label{EqPrice}
\phi_{g+1}-\phi_{g}
&= \int \left(\frac{\mathcal{F}(\phi)}{\langle\mathcal{F}\rangle_g}-1\right) \phi \ P_g(\phi) \ \mathrm{d}\phi=\text{cov}_g(\phi,\mathcal{F}(\phi)/\langle\mathcal{F}\rangle_g).
\end{align}
That is, the change equals the covariance of genotype $\phi$ with the relative fitness, $\mathcal{F}(\phi)/\langle\mathcal{F}\rangle_g$. This change can then be rewritten in terms of the population's variance $V_g=\text{var}_g(\phi)$ as,
\begin{align}\label{EqFisherKeyInsight}
\phi_{g+1}-\phi_{g}
&= V_g \, f_g \quad\text{where}\quad f_g=V_g^{-1}\text{cov}_g(\phi,\mathcal{F}/\langle\mathcal{F}\rangle_g).
\end{align}
Namely, $f_g$ is the slope of the linear regression of the relative fitness on the genotype, $\phi$. As \cite{Frank2025} notes, in the limit where population variance becomes small, this linear regression converges to the kind of gradient vector commonly used in optimization algorithms. This derivation provides a first-principles route to what is often termed the Secondary Theorem of Natural Selection \citep{Robertson1968}. It captures a key Fisherian insight:\footnote{Recall from footnote~\ref{FnAdditiveVariance} that because we are in an asexual context, there is no distinction between the additive variance and the total variance.} A population's rate of adaptive change is directly proportional to its current amount of genetic variation.\footnote{Interpreted carefully, this is a ceteris paribus claim, or at least one thing must be held equal. Namely, it assumes the average fitness gradient, $f_g$, remains roughly constant as one scales the population's variance, $V_g$. For instance, this is not the case in our maze example. Consider a population in a long U-shaped hallway making steady upwards progress at both ends. Although the variance keeps increasing, the average fitness gradient correspondingly decreases keeping the rate of upwards progress constant. However, the following logic shows some important cases in which $f_g$ does remain roughly constant as $V_g$ increases. If $P_g(\phi)$ is a Gaussian distribution (with any variance) then it follows from \citeauthor{Lande1976}'s (\citeyear{Lande1976,Lande1979}) work in continuous quantitative genetics that $f_g=\nabla_{\langle\phi\rangle}\log\langle\mathcal{F}\rangle_g$ is the gradient of the log mean fitness. This means that the underlying fitness landscape, $\mathcal{F}(\phi)$, is first convolved with (or smeared out by) the population's genotype distribution, $P_g(\phi)$. In effect, this will wash out any high-frequency ruggedness in the fitness landscape yielding a smoothed log-fitness landscape, $\log\langle\mathcal{F}\rangle_g$. The average fitness gradient, $f_g$, will be roughly variance-independent if $V_g$ is small compared to the curvature scale of this smoothed landscape. By Stein's Lemma, we also have that $\nabla_{\langle\phi\rangle}\langle\mathcal{F}\rangle_g=\langle\nabla\mathcal{F}\rangle_g$ can be understood as smoothing the fitness gradient. This Gaussian smoothing technique forms the mathematical foundation of many modern Evolutionary Strategies \citep{wierstra2014natural,salimans2017evolution} although it is typically approached from the direction of information geometry. \label{FnCeterisParibus}} In effect, the population's variance, $V_g$, provides the ``fuel'' upon which selection acts and $f_g$ indicates the average direction of the fitness gradient. More technically, $V_g$ acts as both a pre-conditioner and a learning rate for the evolutionary dynamics. This classical result and its algorithmic interpretation will play a significant role throughout the rest of this paper.

\begin{figure}[t!]
\centering
\includegraphics[width=0.95\textwidth]{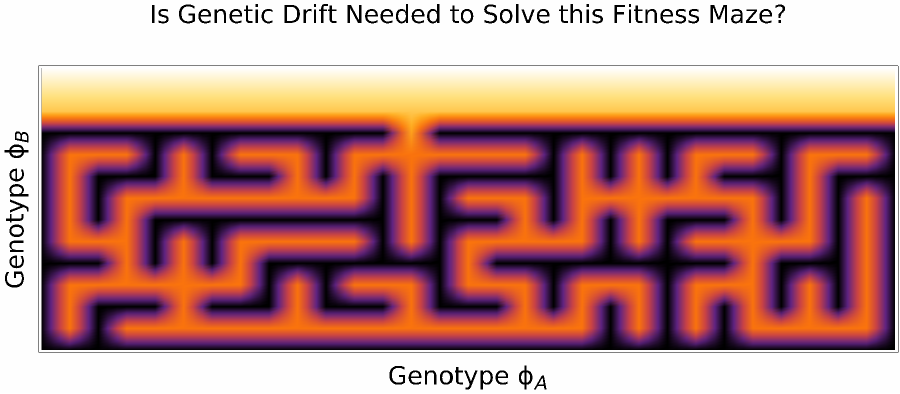}
\caption{A maze-like fitness function, $\mathcal{F}(\phi)$, is plotted in a 2D genotype space ($\phi_A, \phi_B$). The dark regions represent zero-fitness boundaries enclosing flat corridors of constant fitness. There is a single exit at the top leading to an unobstructed region of higher fitness. Suppose that the initial population is localized in a remote corner of the maze. One might wonder whether random genetic drift is required to solve this maze. A Wright-like Darwinian Lineage Simulation (see Theorem~\ref{ThmDLSUpdatePackage}) will eventually solve this maze via the stochastic wandering of genetic drift. Indeed, almost every lineage will eventually wander its way out of this maze (some faster than others). Perhaps surprisingly, Fisher's model of deterministic mass selection, Eq.~\eqref{EqNonStochGenUpdateRule}, will also eventually solve this maze \textit{without genetic drift}, see Sec.~\ref{SecGenUpdate}. It does so via a combination of strong selection at the walls/dead-ends and blind mutation pressure (which acts as a deterministic diffusion process expanding the population's variance along neutral corridors). Indeed, because of the asexual Fisher-Wright equivalence the deterministic Fisherian dynamics will solve this maze exactly as fast as the fastest Darwinian Lineage Simulation. So was the first individual to break out of this maze (and their many, many children) aided by genetic drift? From a lineage perspective, yes. From a population perspective, no.}\label{FigMaze}
\end{figure}

To illustrate the power of these deterministic Fisherian dynamics, consider the maze-like fitness environment depicted in Fig.~\ref{FigMaze}. The fitness function, $\mathcal{F}(\phi)$, is constant within the maze's corridors but drops rapidly to zero at its walls. There is a single exit at the top of the maze leading to an unobstructed region of higher fitness. Suppose that the initial population, $P_0(\phi)$, is localized in a remote corner of the maze. As Eq.~\eqref{EqNonStochGenUpdateRule} is repeatedly applied, the evolutionary dynamics proceeds deterministically; no dice are rolled. While the floor of the maze is flat there is a severe fitness differential at the boundaries. Mutations that push individuals into the walls result in immediate death. As dictated by the Price equation, Eq.~\eqref{EqPrice}, this asymmetric survival creates a covariance between genotype and fitness and thereby causes the population's mean genotype to shift. A constant outwards mutation pressure gradually spreads the genotype distribution along the corridors cut off by harsh selection at the boundaries. Eventually, this spreading population reaches the exit at a time predetermined by the population's starting location and the mutation rate. This flatly contradicts Wright's above-discussed argument that strictly deterministic dynamics cannot cross plateaus without genetic drift. 

In closing, let us next make a small improvement to the generational update rule by adding intra- and inter-generational stochasticity. Different individuals within the same generation experience different fitness pressures: even identical twins might, by chance, encounter different amounts of food and different predator exposure. Alternatively, there may be a complex genotype-to-phenotype mapping such that a single genotype can express different phenotypes depending on external factors. To capture all of this intra-generational variability, we add an index $\eta$ to the fitness function. 

Similarly, we add an index $\lambda$ to capture inter-generational variability in the fitness function. For instance, one might consider a species of birds where each generation is born into a different environment (e.g., desert, swamp, city) and so faces different predators (e.g., snakes, alligators, cars). Alternatively, consider a linguistic example: your parents' generation had to learn German, whereas your generation had to learn English. In \citeauthor{SeaScapes2009}'s (\citeyear{SeaScapes2009}) terminology we are dealing with a ``seascape'' rather than a landscape. Importantly, incorporating this stochastic $\lambda$-index into our fitness function, $\mathcal{F}_{\lambda,\eta}(\phi)$, does not violate our Fisherian credentials. Fisher explicitly built his theories around fluctuating environments, arguing that the constant ``deterioration of the environment'' is a ubiquitous driver of major adaptive shifts \citep{Fisher1950}. 

We can easily generalize Eq.~\eqref{EqNonStochGenUpdateRule} to account for both kinds of variability ($\eta$ and $\lambda$). All we have to do is use a stochastic fitness function, $\mathcal{F}_{\lambda,\eta}(\phi)$, and average over the intra-generational variability, $\eta$, as follows:
\begin{align}\label{EqStochasticGenUpdateRule}
\nonumber
&\text{Stochastic Generational Update Rule:}\\
\nonumber
&P_{g+1}(\phi) \propto \int p(\eta\vert\phi,\lambda) \ \left(\mathcal{F}_{\lambda,\eta}(\phi) \ P_{g}(\phi)\right) \, \ast \, \mathcal{N}(\phi;0,\mu^2 \, I) \ \mathrm{d} \eta,\\
&\qquad\quad \ \propto \left(\overline{\mathcal{F}_{\lambda}}(\phi) \ P_{g}(\phi)\right) \, \ast \, \mathcal{N}(\phi;0,\mu^2 \, I).
\end{align}
Comparing this rule with Eq.~\eqref{EqNonStochGenUpdateRule}, we see that it has the same form, except that the original fitness function, $\mathcal{F}(\phi)$, is replaced by the $\lambda$-indexed averaged fitness function, $\overline{\mathcal{F}_\lambda}(\phi)\coloneqq\int p(\eta\vert\phi,\lambda) \ \mathcal{F}_{\lambda,\eta}(\phi) \ \mathrm{d} \eta$. To reduce notational clutter, we now drop the overline on $\mathcal{F}_{\lambda}$ and implicitly assume that intra-generational variation has already been averaged out.

\subsection{Decomposition Into Sub-Populations: Analytical Bookkeeping}\label{SecLineageVsPopulation}

Thus far we have defined a stochastic generational update rule for the total population's genotype distribution along purely Fisherian lines. However, tracking the total population in this way is computationally difficult: directly evolving a complicated population-level distribution via Eq.~\eqref{EqStochasticGenUpdateRule} over many generations is often intractable. Fortunately, Fig.~\ref{FigGaussGauss}a illustrates how we can greatly simplify this task (without approximation) by dividing the total population (blue) into a ``super-distribution'' of sub-populations (orange). Formally, this decomposition corresponds to writing $P_g(\phi)$ as a convex combination of sub-distributions, $p_g(\phi\mid\alpha)$, indexed by $\alpha$ as follows:
\begin{align}
P_g(\phi)=\frac{1}{N_g}\int \omega_g(\alpha)\,p_g(\phi\mid\alpha)\,\mathrm{d}\alpha,
\end{align}
with $\omega_g(\alpha)\geq 0$, corresponding to the demographic weight of each sub-population (e.g., its size). Each sub-distribution, $p_g(\phi\mid\alpha)$, is normalized such that taking $N_g=\int \omega_g(\alpha)\,\mathrm{d}\alpha$ ensures that $P_g(\phi)$ is also normalized. There is an important sense in which these sub-populations are exactly like Wright's demes: there is absolutely no gene flow between them. Thus, any well-adapted gene complexes will be shielded from the dangers of sexual recombination (there simply is none). Unlike Wright's physically isolated demes, however, our decomposition of the total population into sub-populations carries no special demographic ontology. For us, this decomposition is merely an analytical bookkeeping device.

\begin{figure}[t!]
\centering
\includegraphics[width=0.45\textwidth]{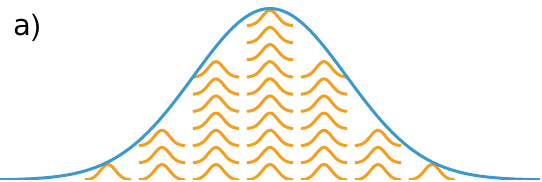}
\includegraphics[width=0.45\textwidth]{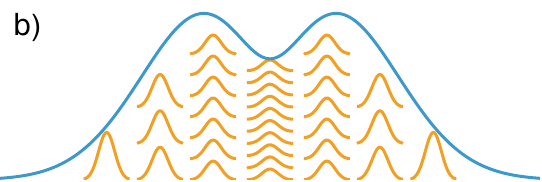}
\caption{An illustration of our ability to track the total population's evolution by independently evolving many sub-populations. a) The total population (blue) can be represented as a ``super-distribution'' of sub-populations (orange). b) Evolving each of these sub-populations (orange) independently and then reassembling them (weighted by their average fitness) exactly recovers the evolved total population (blue).}\label{FigGaussGauss}
\end{figure}

This decomposition is incredibly useful because it gives us a more efficient way to apply the above-discussed Fisherian dynamics to the total population. Importantly, up to normalization, the generational update rule, Eq.~\eqref{EqStochasticGenUpdateRule}, is linear in $P_g(\phi)$. This follows from our assumption that the fitness function, $\mathcal{F}_{\lambda}(\phi)$, is frequency-independent. It follows from this linearity that the evolution of the total-population can be exactly decomposed into the independent evolution of each of the sub-populations, see Fig.~\ref{FigGaussGauss}. Namely, one can apply the generational update rule to each sub-population and then reassemble these evolved sub-populations into the evolved total population. Formally, this means that:
\begin{align}
P_{g+1}(\phi)=\frac{1}{N_{g+1}}\int w_{g+1}(\alpha)\,p_{g+1}(\phi\mid\alpha)\,\mathrm{d}\alpha,
\end{align}
where $p_{g+1}(\phi\mid\alpha)$ is the result of applying Eq.~\eqref{EqStochasticGenUpdateRule} to $p_{g}(\phi\mid\alpha)$. Importantly, when we reassemble these sub-populations we have to use their new demographic weights,
\begin{align}
\omega_{g+1}(\alpha)=\langle \mathcal{F}_\lambda\rangle_{g,\alpha} \ \omega_g(\alpha) \quad\text{where}\quad \langle\mathcal{F}_\lambda\rangle_{g,\alpha}=\int \mathcal{F}_\lambda(\phi) \, p_g(\phi\mid\alpha)\,\mathrm{d}\phi.
\end{align}
Those sub-populations which have a higher average fitness are disproportionately represented in the total population, see Fig.~\ref{FigGaussGauss}b. 

Over time, the evolution of each sub-population gives rise to a corresponding \textit{Darwinian lineage}. By tracking an ensemble of these Darwinian lineages, one can exactly reconstruct the total population's evolutionary history as follows. The change in a sub-population's demographic weight across multiple generations is given by:
\begin{align}\label{EqLineageWeight}
\frac{\omega_{T}(\alpha)}{\omega_{t}(\alpha)}
= \prod_{g=t}^{T-1}\langle \mathcal{F}_{\lambda(g)}\rangle_{g,\alpha }
=\exp\!\left(\sum_{g=t}^{T-1}\log\langle \mathcal{F}_{\lambda(g)}\rangle_{g,\alpha}\right).
\end{align}
That is, a sub-population's demographic weight is determined by the accumulated average log-fitness of its lineage (the sub-population extended over time). Reassembling these Darwinian lineages with the appropriate weightings allows us to exactly recover every aspect of the total population's evolutionary history. For every $t$, we have:
\begin{align}
P_{t}(\phi)=\frac{1}{N_{t}}\int w_{t}(\alpha)\,p_{t}(\phi\mid\alpha)\,\mathrm{d}\alpha,
\end{align}
Hence, by simulating an ensemble of Darwinian lineages and then reassembling them one can recover the total population's Fisherian dynamics exactly (with no approximation). This result forms the conceptual backbone of the asexual Fisher-Wright equivalence which we are building up to. Importantly, this claim holds true no matter how one chooses to decompose the total population up into sub-populations (e.g., Gaussian or otherwise). This complete flexibility in how we do our bookkeeping will play a significant role in our upcoming derivation of advanced optimizers from evolutionary first principles.

\subsection{The Gaussian Lineage Propagation Theorem}\label{SecGaussianUpdateRules}
We shall next make a specific bookkeeping choice. We shall choose to decompose the total population, $P_g(\phi)$, up into Gaussian sub-populations, $p_g(\phi\mid\alpha)$. Let us focus on one of these sub-populations in particular (henceforth dropping its lineage index, $\alpha$) and investigate how it evolves under the above-discussed Fisherian dynamics. As we shall now see, when applied to a small Gaussian sub-population the generational update rule, Eq.~\eqref{EqStochasticGenUpdateRule}, takes on a simple computationally efficient form.
\begin{theorem}[Gaussian Lineage Propagation]\label{ThmGaussianLineageProp}
Suppose that $p_{g}(\phi)$ is a Gaussian distribution with mean, $\phi_{g}$, and variance, $V_{g}\succeq0$. Assume further that the scale of its variance, $V_{g}$, is small compared to the scale at which the log-fitness function, $\log(\mathcal{F}_\lambda)(\phi)$, varies in the neighborhood of $\phi_{g}$. Concretely, we require two things. Firstly, $\log(\mathcal{F}_\lambda)(\phi)$ must be approximately quadratic with Hessian, $H_g$, within a region of several standard deviations of $\phi_{g}$ (as measured by $V_g$). Secondly, we must have $V_{g}^{1/2} \ H_g \ V_{g}^{1/2}\preceq I$. Then, to a good approximation, $p_{g+1}(\phi)$ is also Gaussian with the following mean and variance:
\begin{align}
\nonumber
&\text{Full Update:}\\
\label{EqGenUpdateRulePhiFull}
&\phi_{g+1}
=\phi_{g} + (V_g^{-1} - H_g)^{-1} \, \nabla \log(\mathcal{F}_\lambda)(\phi_{g}),\\
\label{EqGenUpdateRuleSigmaFull}
&V_{g+1}= (V_{g}^{-1}- H_g)^{-1} + \mu^2 \, I. 
\end{align} 
If, moreover, $V_{g}$ is sufficiently small with respect to the Hessian such that $V_{g} \ H_g\ll I$, then:
\begin{align}
\nonumber
&\text{Leading-Order Update:}\\
\label{EqGenUpdateRulePhi}
&\phi_{g+1}
=\phi_{g} + V_g \ \nabla \log(\mathcal{F}_\lambda)(\phi_{g}),\\
\label{EqGenUpdateRuleSigma}
&V_{g+1}= V_{g} + \mu^2 \, I.
\end{align}
In the isotropic case, $V_{g}=\sigma_g^2\,I$, this further reduces to:
\begin{align}
\nonumber
&\text{Isotropic Update:}\\
\label{EqGenUpdateRulePhiIso}
&\phi_{g+1}=\phi_{g}+\sigma_g^2\nabla \log(\mathcal{F}_\lambda)(\phi_{g}),\\
\label{EqGenUpdateRuleSigmaIso}
&\sigma_{g+1}^2=\sigma_g^2+\mu^2.
\end{align}
\end{theorem}
\begin{proof}
It follows from the sub-population's variance, $V_{g}$, being small that $\mathcal{F}_\lambda(\phi)$ can be well-approximated by a Gaussian distribution within a region of several standard deviations of $\phi_{g}$ (as measured by $V_g$). Together with our assumption that $p_{g}(\phi)$ is Gaussian it then follows from the form of Eq.~\eqref{EqStochasticGenUpdateRule} that $p_{g+1}(\phi)$ is approximately Gaussian as well.\footnote{The fact that $p_{g+1}(\phi)$ is normalizable follows from $V_{g}^{1/2} \ H_g \ V_{g}^{1/2}\preceq I$. Concretely, one then has $(V_{g}^{-1}- H_g)^{-1}\succeq0$ such that $V_{g+1}\succeq0$.} It is then easy to calculate the mean and variance of $p_{g+1}(\phi)$ yielding the Full Update described above. To get to the Leading Order Update, note that $(V_{g}^{-1}- H_g)^{-1}=(I-V_g\,H_g)^{-1} \, V_g$.
\end{proof}
\noindent This theorem says that if $p_{g}(\phi)$ is Gaussian with a small enough variance, then $p_{g+1}(\phi)$ will be Gaussian as well (at least approximately). Three versions of the resulting update equation are then provided under additional small and isotropic variance assumptions. Insofar as $p_{g+1}(\phi)$ also has a small enough variance, one can apply these update rules again to find the mean and variance of the subsequent generation, $p_{g+2}(\phi)$.

It should be stressed that the dynamics described by the above theorem is still purely Fisherian. Although we have divided the total population up into genetically isolated sub-populations, these are still evolving deterministically without genetic drift. Moreover, it should be noted that the leading order update rule conforms exactly to Eq.~\eqref{EqFisherKeyInsight} and our earlier discussion of \citeauthor{Fisher1930}'s (\citeyear{Fisher1930}) key insight.\footnote{Of course, the full update rule also conforms with Eq.~\eqref{EqFisherKeyInsight} after one notes the subtlety raised in footnote~\ref{FnCeterisParibus}. Namely, the average fitness gradient $f_g$ is not fully independent of the population's variance. Indeed, by factoring $V_g$ out on the left one finds that $f_g=(I-H_g\,V_g)^{-1}\nabla \log(\mathcal{F}_\lambda)(\phi_{g})$.} Here the gradient of the log-fitness function, $\nabla \log(\mathcal{F}_\lambda)(\phi_{g})$, indicates the direction of selection. This is the direction in genotype space, $\phi\in\mathbb{R}^N$, that most quickly increases this generation's fitness function, $\mathcal{F}_\lambda(\phi)$. The lineage's current variance, $V_g$, then provides the ``fuel'' upon which this selection gradient acts. Importantly, this variational fuel can be anisotropic; selection acts more strongly in the directions in which the lineage has more variance. Correspondingly, if there is no variance in a particular direction, selection cannot act in that direction. In this way, $V_g$ acts as both a pre-conditioner and a learning rate for the evolutionary dynamics.

\begin{figure}[t!]
\centering
\includegraphics[width=0.95\textwidth]{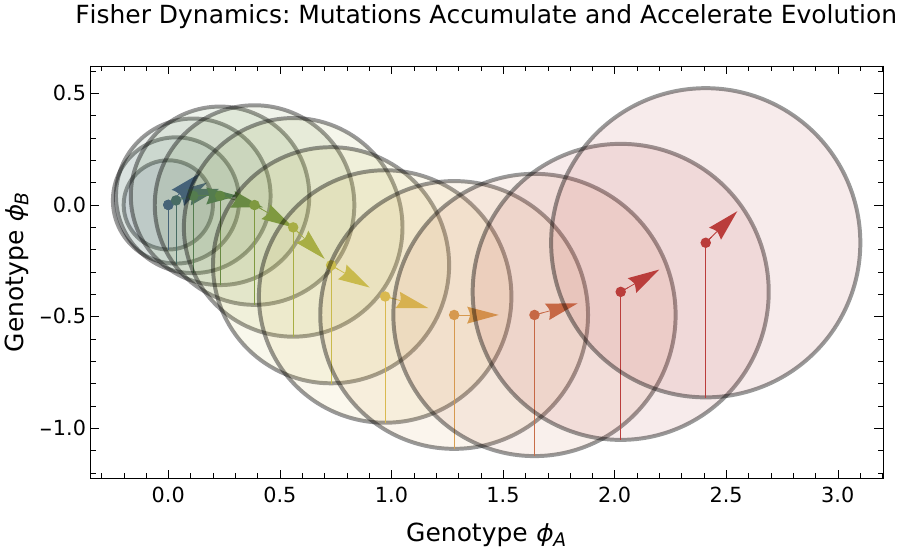}
\caption{The trajectory of a Darwinian lineage is plotted in a 2D genotype space ($\phi_A, \phi_B$) without down-sampling. Each shaded disk represents the lineage's genotype distribution, $p_g(\phi)$, at a specific generation. The point at the center of each disk is the mean genotype, $\phi_g$, while the radius represents its standard deviation, $\sigma_g$. This figure uses the isotropic update rule from Theorem~\ref{ThmGaussianLineageProp}. The arrows indicate the gradient of the log-fitness function, $\nabla\log(\mathcal{F}_{\lambda})(\phi_g)$, at each generation. Notice that even as these gradients remain the same size, the rate of genotype change, $\Delta \phi_g = \sigma_g^2\nabla\log(\mathcal{F}_{\lambda})(\phi_g)$, increases. This is because (in alignment with Fisher's key insight) the lineage's current amount of variation, $\sigma_g^2$, acts as its learning rate. This acceleration will continue at least until $\sigma_g$ becomes comparable to the curvature scale of $\log(\mathcal{F}_\lambda)$. The need to control the variance of the lineage that we track motivates the down-sampling procedures introduced in Sec.~\ref{SecDownSampling}.}\label{FigFisherDyn}
\end{figure}

To illustrate this Fisherian dynamics, let us turn to Fig.~\ref{FigFisherDyn} which shows the result of repeatedly applying Eqs.~\eqref{EqGenUpdateRulePhiIso} and~\eqref{EqGenUpdateRuleSigmaIso} to an individual so as to find the genotype distribution of their great-\dots-grandchildren. Beginning at generation $g=0$ with an individual of genotype $\phi_{0}$, their genotype distribution, $p_{0}(\phi)$, is exactly localized at $\phi_{0}$. Hence, the initial variance is zero: $\sigma^2_{0}=0$. Applying Eq.~\eqref{EqGenUpdateRulePhiIso}, the mean genotype of the next generation does not change at all, $\phi_{1}=\phi_{0}$. Nothing happens because there is no variation for selection to act upon.\footnote{Note that the Full Update also works in this way: $V_g=0$ implies $(V_{g}^{-1}- H_g)^{-1}=0$ } Applying Eq.~\eqref{EqGenUpdateRuleSigmaIso}, however, shows that the next generation does pick up some variance, $\sigma^2_{1}=\mu^2$, due to mutations. Now that there is some variance in the population, there is something for selection to act upon. At generation $g=2$, the new mean genotype is
$\phi_{2}=\phi_{1}+\mu^2\nabla \log(\mathcal{F}_{\lambda_1})(\phi_{1})$.
Moreover, mutation effects compound such that we now have $\sigma^2_{2}=2\mu^2$. In line with our discussion of Fisher's key insight, this increased variance at $g=2$ yields increased selection pressure as we move on to generation $g=3$. The new mean genotype is then
$\phi_{3}=\phi_{2}+2\mu^2\nabla \log(\mathcal{F}_{\lambda_2})(\phi_{2})$.
Moreover, mutation effects compound further such that we have $\sigma^2_{3}=3\mu^2$. 

Several things are going on here that ought to be unpacked. Note that the gradient of the log-fitness function, $\nabla\log(\mathcal{F}_{\lambda})$, is evaluated at a different location at each update, i.e., at $\phi_{1}$, then at $\phi_{2}$, and so on. Every time we take a step up the fitness landscape, we find new ground beneath our feet because we are now at a different location. Moreover, the fitness landscape itself (or rather seascape) changes at each update as its $\lambda$-index changes, i.e., it was $\lambda_1$, then $\lambda_2$, and so on. It is this inter-generational variability of the fitness landscape that gives us \textit{stochastic} gradient ascent. On average, however, any alignment in these $\lambda$-indexed log-fitness gradients will conspire to move the lineage mean genotype in the average direction of increasing log-fitness,
\begin{align}
\langle \nabla\log(\mathcal{F}_\lambda)\rangle
\coloneqq 
\smallint p(\lambda) \ \nabla\log(\mathcal{F}_\lambda) \ \mathrm{d}\lambda,
\end{align}
until it reaches an equilibrium characterized by $\nabla \langle \log(\mathcal{F}_\lambda)\rangle=0$. Importantly, however, the conditions underwriting $p(\lambda)$ might themselves shift, shattering this equilibrium.

There is of course a problem with repeatedly applying either the leading order or isotropic update rules in this way;\footnote{In the full update case, it is not a foregone conclusion that the population's variance will grow uncontrollably and thereby break Gaussianity. Indeed, without mutations the inverse variance evolves as $V_{g+1}^{-1}= V_{g}^{-1}- H_g$, matching a result from \cite{Otwinowski2020}. Near a fitness peak, $H_g$ is negative definite such that $V_g^{-1}$ increases and hence the variance decreases. Combined with the tendency of the $\mu^2I$ term to increase variance, the lineage may reach mutation-selection balance.} Namely, the sub-population's variance will continually increase until it violates the small-variance assumption which underwrites our derivation of these update rules (Theorem~\ref{ThmGaussianLineageProp}). For instance, recall our maze example from Fig.~\ref{FigMaze}. This breakdown occurs when the sub-population's variance becomes comparable to the width of the maze's corridors. Beyond this, the genotype distribution will no longer be approximately Gaussian. This is what we were hinting at above by suggesting that the genotype distribution will eventually become ``too unwieldy to continue tracking efficiently''. To address this issue, Sec.~\ref{SecDownSampling} will now develop a method of ``down-sampling'' that gives us full control over the variance of the tracked population. Put simply, at any point we can choose to continue tracking only a subset of the current population. This allows us to manually decrease $V_{g}$ so that we can continue using the update rules from Theorem~\ref{ThmGaussianLineageProp}. The cost, however, is the introduction of a certain amount of genetic drift.

\subsection{Controlling Variance with Random Down-Sampling: Genetic Drift}\label{SecDownSampling}
The gradient-based update rules derived in Theorem~\ref{ThmGaussianLineageProp} are valid only while the small-variance assumption holds, yet these same rules also enforce a monotonic increase in variance as mutations accumulate. This motivates the need to control the variance of the lineage we are tracking so that we can continue to apply the update rules in Theorem~\ref{ThmGaussianLineageProp}. We need a principled way to limit our attention to a subset of the lineage we are tracking. The simplest option is to randomly pick an individual from $p_{g}(\phi)$ and then restart the simulation by focusing only on that individual's descendants.

One issue, however, with down-sampling to an individual is that it resets the lineage's variance to zero. It may then take several generations for the lineage's variance (and hence its effective learning rate) to return to its ideal range: large enough to be fast, yet small enough to satisfy Theorem~\ref{ThmGaussianLineageProp}. Fortunately, our down-sampling procedure need not be as drastic as starting over from an individual. Indeed, as we now show, we can gain full control over the lineage's variance by exploiting a notable property of Gaussian distributions. Namely, any Gaussian distribution can be viewed as a Gaussian ``super-distribution'' of Gaussian sub-populations (see Fig.~\ref{FigGaussGauss}a). Using this fact, we can limit our attention to a random sub-population and continue tracking the lineage from there.

For the sake of illustration, suppose we have a large collection of people whose heights follow a Gaussian distribution (e.g., the blue curve in Fig.~\ref{FigGaussGauss}a) with variance $\sigma_\text{Blue}^2$. We can always partition this large population into many smaller sub-populations, each with its own Gaussian height distribution (e.g., the orange curves in Fig.~\ref{FigGaussGauss}a) and with smaller variance, $\sigma_\text{Orange}^2\leq\sigma_\text{Blue}^2$. Suppose each sub-population computes its mean height and submits that value to a collective database. Remarkably, the distribution of these means is also Gaussian, with some variance $\sigma_\text{Mean}^2$. The key point is that the variance reduction we achieve, $\sigma_\text{Blue}^2-\sigma_\text{Orange}^2$, is exactly equal to the variance of the means: $\sigma_\text{Mean}^2=\sigma_\text{Blue}^2-\sigma_\text{Orange}^2$.

The same fact holds in higher dimensions (e.g., in genotype space, $\phi\in\mathbb{R}^N$) and for anisotropic Gaussian distributions. Suppose that at generation $g$ the lineage's genotype distribution, $p_g(\phi)$, is Gaussian with mean $\phi_g$ and covariance $V_g$. After applying one of the generational update rules from Theorem~\ref{ThmGaussianLineageProp}, we obtain a new Gaussian distribution, $\tilde{p}_{g+1}(\phi)$, with mean $\tilde\phi_{g+1}$ and covariance $\tilde V_{g+1}$. The tildes indicate that down-sampling has not yet been applied. Now let $W_g\succeq0$ denote the covariance reduction which we would like to now impose at the down-sampling step. We can decompose $\tilde{p}_{g+1}(\phi)$ into many Gaussian sub-populations, each with a variance of,
\begin{align}\label{EqDownSampleTilde}
V_{g+1}=\tilde V_{g+1}-W_g,
\end{align}
and with means distributed as,
\begin{align}
\phi_{g+1}=\tilde\phi_{g+1}+\xi_g,\qquad \xi_g\sim\mathcal{N}(0,W_g).
\end{align}
Now we can simply pick one of these lower-variance sub-populations at random and continue tracking its evolution, e.g., by applying Theorem~\ref{ThmGaussianLineageProp} before down-sampling again. In particular, we can use this down-sampling technique at every step to ensure that the theorem's small-variance condition remains satisfied. This random down-sampling procedure is demonstrated in Fig.~\ref{FigDownSampling}. At each step the lineage's genotype distribution shifts and widens before we restrict our attention to a randomly selected sub-population. The resulting random offsets at each step amount to a kind of genetic drift.

\begin{figure}[t!]
\centering
\includegraphics[width=0.9\textwidth]{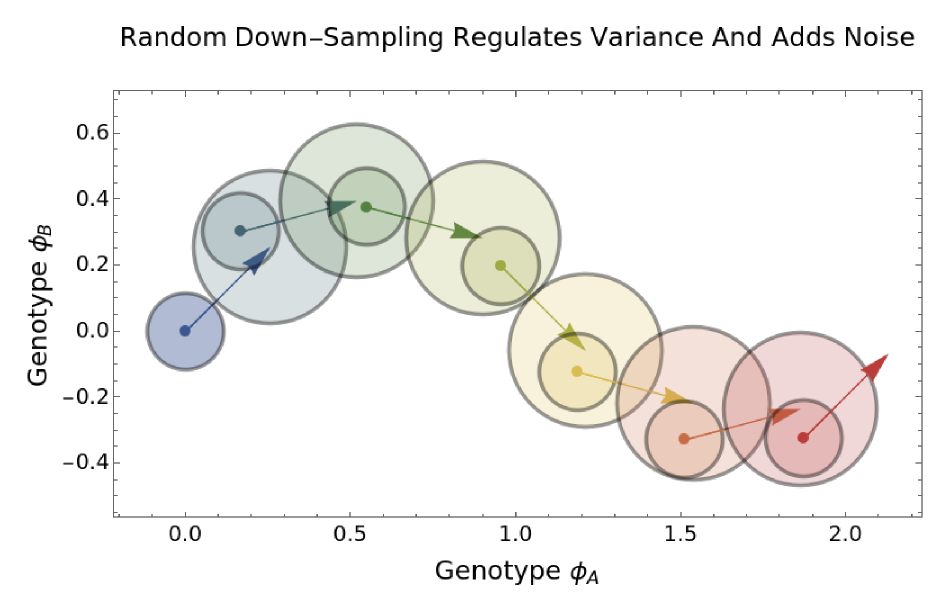}
\caption{The trajectory of a Darwinian lineage is plotted in a 2D genotype space ($\phi_A, \phi_B$) with random down-sampling. The left-most disk at $(0,0)$ represents the initial Gaussian lineage distribution, $p_0(\phi)$, at generation $g=0$. The point at the center of this disk is the mean genotype, $\phi_0$, while its radius represents the lineage's standard deviation, $\sigma_0$. This figure uses the isotropic DLS update rule from Theorem~\ref{ThmDLSUpdatePackage}. Accordingly, the arrow from $(0,0)$ shows the selection-driven mean shift, $\Delta \phi_g = \sigma_g^2\nabla\log(\mathcal{F}_{\lambda})(\phi_g)$. Prior to down-sampling, mutations increase the lineage's variance at each generation as $\tilde{\sigma}_{g+1}^2=\sigma_g^2+\mu^2$. The variance of the subsequent generation is hence represented by a larger disk, before we randomly down-sample to a Gaussian sub-population for continued tracking. This randomly selected sub-population generically has an offset mean, so a noise term $\xi_g$ must appear in the mean's update rule.}\label{FigDownSampling}
\end{figure}

We are now in a position to revise Theorem~\ref{ThmGaussianLineageProp} in light of our new down-sampling procedure. As we shall soon discuss, this amounts to a formal unification of Fisher and Wright.
\begin{theorem}[The DLS Update Rules]\label{ThmDLSUpdatePackage}
Suppose that the genotype distribution of the lineage we are tracking, $p_g(\phi)$, is Gaussian with mean $\phi_g$ and variance $V_g\succeq0$. Suppose further that the scale of its variance, $V_{g}$, is small compared to the scale at which the log-fitness function, $\log(\mathcal{F}_\lambda)(\phi)$, varies in the neighborhood of $\phi_{g}$. Concretely, we require two things.  Firstly, $\log(\mathcal{F}_\lambda)(\phi)$ must be approximately quadratic with Hessian, $H_g$, within a region of several standard deviations of $\phi_{g}$ (as measured by $V_g$). Secondly, we must have $V_{g}^{1/2} \ H_g \ V_{g}^{1/2}\preceq I$. We can apply the generational update rule from Theorem~\ref{ThmGaussianLineageProp}, and then randomly down-sample with a variance reduction of $W_g\succeq 0$. The resulting genotype distribution, $p_{g+1}(\phi)$, is Gaussian to a good approximation with the following mean and variance:
\begin{align}
\nonumber
&\text{Full DLS Update:}\\
&\phi_{g+1}
=\phi_g + (V_g^{-1}-H_g)^{-1}\,\nabla\log(\mathcal{F}_\lambda)(\phi_g) + \xi_g,
\label{EqThm2FullPhi}
\qquad \xi_g\sim\mathcal{N}(0,W_g),\\
&V_{g+1}=(V_g^{-1}-H_g)^{-1} + \mu^2I - W_g.
\label{EqThm2FullV}
\end{align}
If $V_{g}$ is sufficiently small with respect to the Hessian such that $V_{g} \ H_g\ll I$, then,
\begin{align}
\nonumber
&\text{Leading Order DLS Update:}\\
&\phi_{g+1}
=\phi_g + V_g\,\nabla\log(\mathcal{F}_\lambda)(\phi_g) + \xi_g,
\label{EqThm2LeadingPhi}
\qquad \xi_g\sim\mathcal{N}(0,W_g),\\
&V_{g+1}=V_g + \mu^2I - W_g.
\label{EqThm2LeadingV}
\end{align}
Moreover, if we want $V_g=\sigma_g^2I$ to always remain isotropic then we must take $W_g=w_g^2I$ yielding:
\begin{align}
\nonumber
&\text{Isotropic DLS Update:}\\
&\phi_{g+1}=\phi_g + \sigma_g^2\,\nabla\log(\mathcal{F}_\lambda)(\phi_g) + \xi_g,
\label{EqThm2IsoPhi}
\qquad \xi_g\sim\mathcal{N}(0,w_g^2I),\\
&\sigma_{g+1}^2=\sigma_g^2+\mu^2-w_g^2.
\label{EqThm2IsoSigma}
\end{align}
So long as $W_{g}$ and $V_{g+1}$ are both symmetric and positive semi-definite, our choice of $W_g$ at each down-sampling step is unconstrained. By controlling the lineage's variance, $V_g$, in this way, one can repeatedly apply these update rules without violating its small-variance conditions. By running an ensemble of such DLS simulations and reassembling them with the fitness-weighting described in Sec.~\ref{SecLineageVsPopulation}, one can recover the total population's Fisherian dynamics to arbitrary precision. Importantly, this recovery is independent of the specific choices of $W_g$.
\end{theorem}
\begin{proof}
This result is a direct assembly of the ingredients established above. 
\end{proof}
\noindent The strict mathematical coupling between the variance reduction at each step, $W_g$, and the covariance of the associated genetic drift, $\xi_g\sim\mathcal{N}(0,W_g)$, is what we term the \textit{DLS noise relation}. For instance, in the leading order case we have $W_g=\mu^2I-(V_{g+1}-V_g)$, the size and shape of genetic drift is fixed by the mutation rate, $\mu^2$, together with changes in the tracked populations variance, $V_{g+1}-V_g$. As we will discuss further in Sec.~\ref{SecLitReview}, this particular noise structure is not an ad hoc modeling choice; genetic drift must look exactly like this if one's bookkeeping is to be evolutionarily faithful.

The fact that these update equations now feature genetic drift, allows us to view an ensemble of DLS simulations as an asexual realization of \citeauthor{Wright1931}'s (\citeyear{Wright1931,Wright1932}) Shifting Balance Theory.
\begin{enumerate}
    \item Wright's Phase I postulates that small genetically isolated sub-populations undergo random genetic drift across the fitness landscape. Our asexual sub-populations have no gene flow between them and do experience genetic drift (see our discussion of the maze example below). 
    \item Wright's Phase II occurs when, through this stochastic wandering, a lineage discovers the slope of a new, higher-fitness peak. At that point, local mass selection drives the lineage rapidly up the log-fitness landscape. This corresponds to the gradient-driven term in our DLS update rules, where the lineage's variance acts as a pre-conditioner and learning rate, in line with Fisher's key insight, Eq.~\eqref{EqFisherKeyInsight}.
    \item Wright's Phase III proposes that successful sub-populations will become disproportionately represented in the total population. As we have discussed above, Wright's original migration mechanism for this phase has been replaced in our asexual context with clonal interference. The more successful sub-populations will come to dominate the total population by sheer demographic weight. In our model, this is captured by the fitness-weighted reassembly of the ensemble of all Darwinian lineages. Lineages that discover higher-fitness regions will accumulate exponentially larger reassembly weights, exactly as specified in Sec.~\ref{SecLineageVsPopulation}.
\end{enumerate}
Hence, an ensemble of Darwinian Lineage Simulations is an asexual realization of Wright's Shifting Balance Theory. The key difference is that our division of the total population into sub-populations carries no special demographic ontology; it is purely a matter of bookkeeping.

Let us next re-approach the maze-like fitness environment depicted in Fig.~\ref{FigMaze} from our newfound Wrightian perspective. As before, suppose we begin from an individual with genotype $\phi_0$ in a remote corner of the maze. We can follow one of this individual's lineages by repeatedly applying one of the update rules from Theorem~\ref{ThmDLSUpdatePackage}. Indeed, we can proceed at first without down-sampling (choosing $W_g=0$) at least until the population's growing variance begins to approach the walls of the maze. In this time, the population's mean will remain unchanged because the maze floor is effectively flat, $\nabla\log(\mathcal{F}_\lambda)(\phi)=0$, and $W_g=0$ implies $\xi_g=0$, no genetic drift. Beyond this point, however, further evolution without down-sampling would break the Gaussianity of $p_g(\phi)$. Proceeding with $W_g\neq0$, the lineage which we are tracking does experience genetic drift, $\xi_g\neq0$, and so will almost surely wander its way out of the maze eventually.

It should be stressed that from a lineage-level perspective, genetic drift is absolutely required to solve this maze. To see this, consider an alternative down-sampling procedure where instead of choosing to follow a random sub-population, you always choose to follow one of the mode sub-populations. With mode down-sampling, one still has a reduction in the lineages' variance ($W_g\neq0$) but no associated shift in its mean genotype. Because the maze's floor is effectively flat, $\nabla\log(\mathcal{F}_\lambda)(\phi)=0$, no genetic drift means no motion whatsoever and the mode down-sampled lineage will never solve the maze. More generally, this contrast between what random and mode down-sampling can achieve is a useful diagnostic tool: if lineages with mode down-sampling robustly develop some trait $X$, then a selection-only explanation is sufficient to account for this trait. By contrast, if this trait only ever appears in lineages with random down-sampling, then it is revealed that genetic drift must be playing an essential role.

But wait, recall from Sec.~\ref{SecGenUpdate} that this maze was solved \textit{without genetic drift} by Fisher's deterministic dynamics, i.e., repeated application of Eq.~\eqref{EqNonStochGenUpdateRule}. How can an ensemble of randomly drifting lineages reassemble into a perfectly deterministic total population? How can the necessity of randomness for solving the maze in the first case be reconciled with it being unnecessary in the second case? The resolution to this apparent paradox lies in the subtle mathematics of how the deterministic evolution of a whole can look stochastic while we are tracking its parts.\footnote{This line of thinking was originally developed in support of (but stands independent from) the many-worlds interpretation of quantum mechanics, see \cite{SelfLocatingUncertainty}.} In finite biological populations, genetic drift is conventionally understood as an ontological reality---an actual random sampling error which becomes only approximately deterministic for large populations. However, in our continuous analytical framework, the total population is treated as primary and is perfectly deterministic. When we apply down-sampling to track a single Wright-like lineage, we are analytically isolating one specific path through this broader deterministic expansion. An individual in this lineage may wonder: ``Why did my lineage happen to mutate in the particular way that it did?'' Note, however, that every possible pattern of mutations is realized somewhere in our (hypothetical) DLS ensemble. Hence what this individual is really asking is: ``Why am I in the lineage that mutated in this particular way rather than, e.g., that other way?'' This is merely a \textit{self-locating uncertainty}, which is perfectly compatible with the perfect deterministic evolution of the total population. Viewed holistically, no dice are needed to decide which of the two lineages happens; both lineages happened and both contain individuals asking this question. The asexual Fisher-Wright equivalence recovers determinism exactly (not approximately) for the total population. 

A quick comment is here warranted about the status of the approximations made in Theorems~\ref{ThmGaussianLineageProp} and~\ref{ThmDLSUpdatePackage} in order to maintain strict Gaussianity under the small variance condition. The sole function of these approximations is to keep the update rules computationally efficient; they play no role whatsoever in establishing the asexual Fisher-Wright equivalence. Indeed, no matter how one chooses to decompose the total population up into sub-populations (e.g., Gaussian or otherwise) the fitness-weighted reassembly described in Sec.~\ref{SecLineageVsPopulation} is exact over multiple generations. Moreover, it remains exact even as we further sub-divide the sub-populations at every generation (as in the above-described down-sampling process). If one tracks the random motion of all of these Wright-like lineages exactly (including their non-Gaussianities) and reassembles them, then the deterministic Fisherian dynamics of the total population is recovered \textit{exactly}. This is the asexual Fisher-Wright equivalence.

It should be stressed that we are free to do our bookkeeping however we like. Namely, we are free to choose the shape and size of the sub-populations that we down-sample into at each step. Provided that we keep track of them all (using the DLS noise relation) and recombine them appropriately, the ensemble-level equivalence between Fisher and Wright will be preserved. This representational freedom is the engine of the results that follow. In the next section, we show that by making specific strategic choices about how these sub-populations are shaped at each generational step (e.g., elongating them along historical gradients or scaling them to match local curvature) we can derive a wide range of advanced optimization algorithms, all in perfect alignment with the underlying evolutionary dynamics.

\section{Deriving Advanced Optimizers from Evolutionary First Principles}\label{SecLitReview}
In the previous section, we established the Darwinian Lineage Simulation (DLS) framework as an evolutionarily faithful realization of both \citeauthor{Fisher1930}'s (\citeyear{Fisher1930}) Large Population Size Theory and \citeauthor{Wright1931}'s (\citeyear{Wright1931,Wright1932}) Shifting Balance Theory (both adapted to an asexual context). We now pivot from descriptive population genetics to prescriptive algorithm design. To do so, we must briefly bridge the terminological gap between evolutionary biology and optimization theory. In evolution, populations adapt by \textit{ascending} a log-fitness gradient, $\nabla\log(\mathcal{F}_\lambda)$. In contrast, optimization algorithms typically seek to \textit{descend} a cost or loss function. Throughout this section, we will retain our biological notation (tracking the mean genotype, $\phi_g$, and its variance, $V_g$) to emphasize that our derived algorithms are fundamentally evolutionary. To recover standard optimization nomenclature, one need only recognize that descending a cost function is mathematically equivalent to ascending $\log(\mathcal{F}_\lambda)(\phi)$. It should also be stressed that the lineage's variance $V_g$ will play different roles in different contexts; in what follows it may act as a learning rate, a pre-conditioner, or an inverse damping matrix (a region of trust). 

This section will now leverage the immense bookkeeping freedom within the DLS framework to derive a suite of advanced optimization algorithms. In Sec.~\ref{SecDLSNoiseRel}, we introduce the DLS noise relation, demonstrating how pre-scheduled or adaptive changes in $V_g$ mandate that we inject a specific kind of structured noise to maintain an evolutionarily faithful model of genetic drift. Simply by adding DLS noise, we shall reveal that a broad family of battle-tested optimization algorithms are already scientifically valid simulations of Darwinian evolution. These include: Stochastic Gradient Ascent as well as many regularizations/approximations of Newton's method and Natural Gradient Ascent. In Sec.~\ref{SecRecoverAdam}, we apply this framework to the state-of-the-art Adam optimizer, performing the minor surgical modifications which are necessary to render its momentum and adaptive learning rates fully compliant with evolutionary dynamics.

\subsection{Recovering Advanced Optimizers by adding the DLS Noise Relation}\label{SecDLSNoiseRel}
Let us begin by emphasizing an alternative perspective on Theorem~\ref{ThmDLSUpdatePackage}. Rather than first choosing an amount to down-sample by, $W_g$, and then computing the subsequent variance, $V_{g+1}$, one could instead pre-specify one's desired variance schedule, $\{V_g\}$, and then infer how much genetic drift, $\xi_g\sim\mathcal{N}(0,W_g)$, is needed to maintain this schedule. 

For instance, suppose that one wants an evolutionarily faithful implementation of SGA with the classical learning rate schedule of $c/t$ for some $c>0$. By Fisher's key insight, Eq.~\eqref{EqFisherKeyInsight}, selection acts upon variance and hence we must identify this learning rate with the lineage's variance as $\sigma_t^2=c/t$. But this means that in order to slowly decrease the learning rate, the lineage that we are tracking needs to be made of smaller and smaller sub-populations at each generation. Hence, the amount that we down-sample by, $w_t^2$, needs to offset the mutation effects, $\mu^2$, as well as implement our desired variance reduction, $\sigma_{t+1}^2-\sigma_t^2$. Solving Eq.~\eqref{EqThm2IsoSigma} one can find the necessary noise schedule, $w_t^2=\mu^2+c/(t^2+t)$. Notice that in addition to a baseline level of mutation-induced noise, $\mu^2$, there is an additional term, $c/(t^2+t)$, which would be present even in the complete absence of mutations. This term represents the active, deliberate down-sampling required to shrink the lineage and force convergence, acting as an algorithmic ``cooling'' mechanism.

This same logic applies to modern deep learning heuristics, such as linear warmups, but with a strict biological constraint. Because the injected genetic drift must be non-negative, $w_g^2\geq0$, the DLS noise relation dictates that a lineage's variance cannot increase faster than the baseline mutation rate, $\sigma_{g+1}^2\leq\sigma_g^2+\mu^2$. The DLS framework enforces a kind of biological speed limit. An algorithmic ``warm up'' simply corresponds to a deliberate cessation of active down-sampling, $w_g^2=0$, allowing the natural mutation rate to expand the lineage's variance, thereby increasing the algorithm's learning rate.

One other consideration stems from the way that down-sampling works, as depicted in Fig.~\ref{FigGaussGauss}a. Recall from Eq.~\eqref{EqDownSampleTilde} that the amount that one chooses to down-sample by, $W_g$, determines how a Gaussian distribution with a variance of $\tilde{V}_{g+1}$ (after mutation and selection but before down-sampling) will be decomposed into many \textit{equally sized/shaped} Gaussian distributions each with a variance of $V_{g+1}=\tilde{V}_{g+1}-W_g$. It is only at this point that one of these sub-populations is randomly selected, determining $\xi_g$, and thereby fixing $\phi_{g+1}$. The key upshot of this is that one's choice of $W_g$ cannot depend upon $\phi_{g+1}$. Thus, the variance that ultimately results from one's down-sampling, $V_{g+1}$, must be independent of $\phi_{g+1}$ as well as all local data there: e.g., the gradient, the hessian, and the Fisher information matrix at $\phi_g$. Notice that this argument places a constraint on every $V_g$ \textit{except for} $V_0$ which the user might choose to specify directly in terms of local data at $\phi_0$. Overall, this imposes only a minor constraint upon the pre-scheduled or adaptive schemes for $\{V_g\}$ which are evolutionarily compliant. Sec.~\ref{SecRecoverAdam} will demonstrate an implementation of this constraint.

Generalizing these examples, we find that almost any noisy implementation of Stochastic Gradient Ascent (SGA) can be understood as a Darwinian Lineage Simulation (DLS), so long as it obeys the DLS noise relation.
\begin{corollary}[SGA plus the DLS Noise Relation]\label{ThmDLSNoiseRelIso}
Consider a generic implementation of SGA with a variable learning rate, $\sigma_g^2\geq0$, and variable isotropic Gaussian noise, $w_g^2\geq0$. Namely,
\begin{align}
\phi_{g+1}=\phi_g + \sigma_g^2\,\nabla\log(\mathcal{F}_\lambda)(\phi_g) + \xi_g,
\qquad \xi_g\sim\mathcal{N}(0,w_g^2I).
\end{align}
In order for this to be a biologically faithful simulation of evolution (specifically, a DLS) it must satisfy the DLS noise relation for some fixed mutation rate, $\mu^2\geq0$, 
\begin{align}\label{EqDLSNoiseRelIso}
w_g^2=\mu^2-(\sigma_{g+1}^2-\sigma_g^2).
\end{align}
Moreover, the objective function, $\log(\mathcal{F}_\lambda)$, needs to be approximately quadratic with Hessian, $H_g$, within a region of several standard deviations of $\phi_{g}$ (as measured by $\sigma_g^2$) with $\sigma_g^2 H_g\ll I$. Lastly, if $g>0$ then $\sigma_g^2$ must be independent of local data at $\phi_{g}$.
\end{corollary}
\begin{proof}
This is just a restatement of part of Theorem~\ref{ThmDLSUpdatePackage} and the above-discussion.
\end{proof}
\noindent The DLS noise relation can be understood as saying that there is a baseline noise level of $\mu^2$ which increases if the learning rate is decreasing, and decreases if the learning rate is increasing. In effect, hitting the brakes causes noise whereas accelerating decreases noise (up to the maximum possible acceleration, with no noise).

It is instructive to compare this result with recent literature, beginning with \cite{Whitelam2021} who previously connected a certain evolutionary algorithm with SGD. Recall from Sec.~\ref{SecIntro} that they invoke a ``reciprocal evolutionary temperature'', $\beta$, and accept proposed mutations ``with the Metropolis probability $\min(1,\exp(-\beta\Delta U))$, a common choice in the physics literature''. Ultimately, they find a noise level of $\mu^2$ and a learning rate of $\beta\mu^2$. This prior work relies heavily on concepts imported from statistical physics. For them, genetic drift is conceptually identified with thermal noise. By contrast, for us genetic drift comes from drawing a random sample out of a larger distribution, the variances of which are tied to the learning rate by Fisher's key insight, Eq.~\eqref{EqFisherKeyInsight}.

While \citeauthor{Whitelam2021}'s (\citeyear{Whitelam2021}) result mirrors our own in the special case of a constant learning rate ($\sigma_g^2=\text{const.}$), modern optimization techniques near ubiquitously require a decaying learning rate. Under a thermal noise analogy, a decaying learning rate would require the ``temperature'' of the evolutionary environment to magically cool down over time, or else for the species mutation rate to mysteriously decrease on a fixed schedule. Either of these would be ad hoc modifications forced on the biological model to bring it into line with engineering practice. Evolutionary biologists should not have to twist themselves into such thermodynamic knots in order to model genetic drift, and now they don't have to. The DLS framework resolves this gracefully. A decaying learning rate simply reflects the modeler's choice to track a progressively narrower biological lineage via down-sampling. As the sub-population shrinks, Corollary~\ref{ThmDLSNoiseRelIso} dictates the exact injection of genetic drift required to maintain biological fidelity. 

Even when the prior literature properly grounds itself in population genetics \citep{Kucharavy2023}, its results stop at recovering first-order, isotropic SGA with a constant learning rate. Such methods are incapable of passing the Rosenbrock benchmark, see Fig.~\ref{FigRosenbrock}. Understood as a log-fitness landscape, the Rosenbrock test function is a curving ridge with an extreme drop in fitness to either side which leads up to a fitness peak. With a constant isotropic variance, $V_g=\sigma^2 I$, maintaining evolutionary fidelity ($\sigma^2 H_g\ll I$), requires one to take very small steps due to the narrowness of the ridge. Small steps are computationally inefficient, and may result in a mutation-dominated random walk, see the left panel of Fig.~\ref{FigRosenbrock}. In sum: prior results in gradient-based evolution cannot pass this benchmark efficiently. 

\begin{figure}[t!]
\centering
\includegraphics[width=0.95\textwidth]{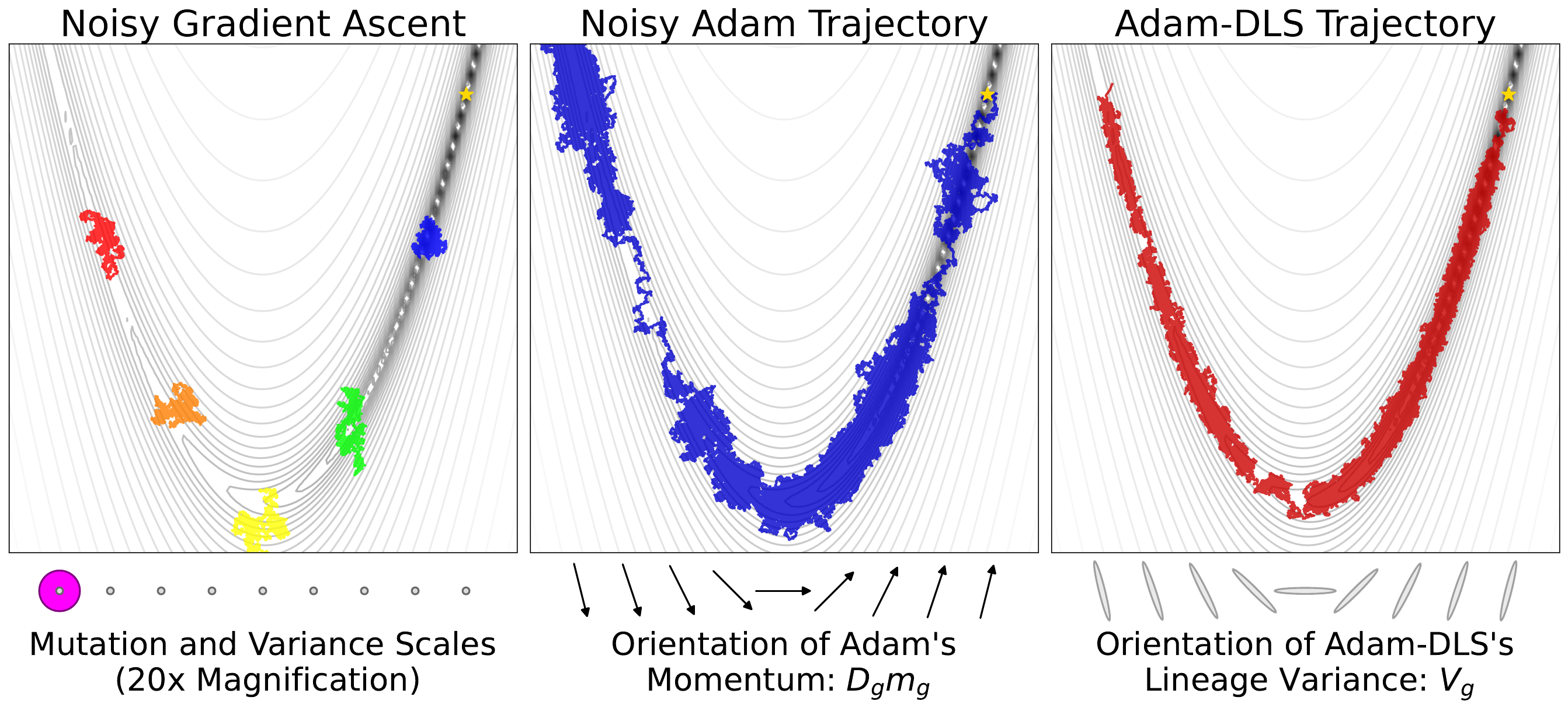}
\caption{The Rosenbrock benchmark (with $a=2$, $b=100$) is attempted by three algorithms: Noisy Gradient Ascent, Noisy Adam, and Adam-DLS. While 
noisy gradient ascent is evolutionarily compliant (it fits the DLS framework) it cannot solve this benchmark efficiently. With an isotropic variance, evolutionary fidelity, $\sigma_g^2 H_g\ll I$, bounds the allowed step size by the high curvature across the ridge. Concretely, at $(-2,4)$ the spectral norm of the Hessian is $\vert\vert H\vert\vert_2=3402$ forcing $\sigma_g^2=0.01/\vert\vert H\vert\vert_2=3\times10^{-6}\ll\mu^2=10^{-4}$. As a result, the lineage experiences almost no selection pressure and hence follows a mutation-dominated random walk. A more sophisticated optimizer (such as Adam) is necessary to solve the benchmark efficiently. While Adam is not a faithful simulation of evolution, Adam-DLS is (see Sec.~\ref{SecRecoverAdam}). Evolutionary compliance, $V_g H_g\ll I$, permits an anisotropic variance which is small across the ridge but large along the ridge. In Adam-DLS, this elongation has two sources: a RMSProp-like diagonal preconditioner, $D_g$, and a rank-1 extension in the (pre-conditioned) direction of historical momentum, $D_g m_g$. The latter term has the exact same orientation as Adam's additive momentum and yields a similar effect. Importantly, Adam-DLS (unlike Adam) is a faithful simulation of evolution. In this figure, both Adam and Adam-DLS take $\alpha=10^{-3}$, $\beta_1=0.99$, $\beta_2=0.999$, $\epsilon=10^{-8}$, and $\mu^2=10^{-4}$. These hyper-parameters maintain $\text{Tr}(V_gH_g)\lesssim0.01$ and $W_g\succeq0$ without ad hoc mutation spikes. A histogram of the $d_g$ scalar is unimodal with a mean of $0.91$ and standard deviation of $28.03$. An interactive reproduction of this benchmark is available at \url{https://github.com/danielgrimmer/adam-dls}.}\label{FigRosenbrock}
\end{figure}

As we shall now discuss, the DLS framework allows us to move beyond this prior limitation into considerations of anisotropic pre-conditioners and second-order geometries. Indeed, evolutionary compliance, $V_g H_g\ll I$, does allow for the lineage's variance to be large precisely in those directions where the fitness landscape's curvature is small. By Fisher's key insight, Eq.~\eqref{EqFisherKeyInsight}, an increased variance along the ridge, means an increased selection pressure up the ridge. Indeed, a wide range of pre-conditioned optimization algorithms are already evolutionarily compliant, once equipped with the DLS noise relation.
\begin{corollary}[Pre-conditioned SGA plus the DLS Noise Relation]\label{ThmDLSNoiseRelLeading}
Consider a generic implementation of a pre-conditioned SGA with a variable pre-conditioner, $V_g\succeq0$, and a variable Gaussian noise, $W_g\succeq0$, both being symmetric and positive semi-definite. Namely,
\begin{align}
&\phi_{g+1}
=\phi_g + V_g\,\nabla\log(\mathcal{F}_\lambda)(\phi_g) + \xi_g,
\qquad \xi_g\sim\mathcal{N}(0,W_g),
\end{align}
In order for this to be a biologically faithful simulation of evolution (specifically, a DLS) it must satisfy the DLS noise relation for some fixed mutation rate, $\mu^2\geq0$,
\begin{align}\label{EqDLSNoiseRel}
W_g=\mu^2I-(V_{g+1}-V_g),
\end{align}
Moreover, the objective function, $\log(\mathcal{F}_\lambda)$, needs to be approximately quadratic with Hessian, $H_g$, within a region of several standard deviations of $\phi_{g}$ (as measured by $V_g$) with $V_g H_g\ll I$.
Lastly, if $g>0$ then $V_g$ must be independent of local data at $\phi_{g}$.
\end{corollary}
\begin{proof}
This is just a restatement of part of Theorem~\ref{ThmDLSUpdatePackage} and the above discussion.
\end{proof}

It is instructive to begin by considering the simplest form of anisotropy: diagonal pre-conditioning. A foundational example in the machine learning literature is the RMSProp algorithm (a precursor to Adam). RMSProp tracks an exponential moving average of the squared gradients for each parameter, $s_{g+1}=\beta_2 s_g+(1-\beta_2) f_g^2$, where $\beta_2\in [0,1)$ controls the decay rate and $f_g^2$ is the element-wise square of the gradient, $f_g=\nabla\log(\mathcal{F}_\lambda)(\phi_g)$. Depending on how $s_g$ is initialized (particularly, if $s_0=0$) it may be a biased estimator of the problem's true second moments. This can be corrected by using $\hat{s}_g = s_g / (1 - \beta_2^{g})$ instead. It should also be noted that in our indexing convention, the quantity $s_g$ represents the average second moments of the \textit{past gradients}, not yet including $f_g^2$. 

At each step, the RMSProp algorithm rescales the current gradient, $f_g$, coordinate-wise by the inverse square root of this average as, 
\begin{align}
V_g \ f_g=\text{diag}\left(\frac{\alpha}{\sqrt{s_{g+1}/(1 - \beta_2^{g+1})}+\epsilon}\right) \ f_g.
\end{align}
for some base learning rate, $\alpha$, with some small $\epsilon$ included for numerical stability. In effect, this pre-conditioning implements a rescaling of each coordinate so as to make the level sets of the objective function more ``rounded'', making local minima more accessible. 

Notice, however, that there is a minor issue with implementing RMSProp within the DLS framework, $V_g$ depends upon $s_{g+1}$ and hence on $f_{g}$. Recall from above that this is not allowed except at $g=0$. The fix, however, is simple:\footnote{In the online learning optimization literature, this ``off-by-one'' shift represents the distinction between predictive and data-dependent regularization \citep{Brendan2017}.} we can instead use
\begin{align}
V_g \ f_g=\text{diag}\left(\frac{\alpha}{\sqrt{s_g/(1 - \beta_2^{g+1})}+\epsilon}\right) \ f_g.
\end{align}
Note that we have not reduced the index on the  $(1 - \beta_2^{g+1})$ factor. One must then initialize $s_0$ to yield the user's specification for $V_0$, the lineage's initial (diagonal) variance. A very natural choice would be $s_0=(1-\beta_2)f_0^2$ from which it follows that $s_0/(1-\beta_2)=s_1/(1-\beta_2^2)=f_0^2$. This initialization for $s_0$ keeps the pre-conditioner perfectly stable through the early updates while every subsequent $s_g/(1 - \beta_2^{g+1})$ mechanically mimics the RMSProp buildup. With this index-shift and initialization, one can effortlessly implement RMSProp in an evolutionarily faithful way. One simply chooses to down-sample so that the lineage's variance always proceeds according to the desired adaptive scheme. Further details regarding RMSProp-DLS can be found in Sec.~\ref{SecRecoverAdam}, just take $\beta_1=0$ for the Adam-like optimizer (Adam-DLS) which we there construct.

Beyond this example, a gold standard for pre-conditioning in optimization theory is Natural Gradient Descent \citep{Amari1998}, where one's pre-conditioner is proportional to the inverse Fisher Information Matrix (FIM).\footnote{Recent work by \cite{Otwinowski2020} has elegantly framed natural selection as natural gradient descent. However, they rely on inventing novel recombination operators to prevent variance collapse in low-mutation regimes. By contrast, our framework proceeds asexually with mutation to derive a flexible model of genetic drift which can recover many battle-tested optimizers including many efficient approximations of natural gradient descent.} However, for modern, over-parameterized neural networks, computing and storing a full $N \times N$ matrix is computationally intractable. Consequently, the field has developed highly efficient methods of approximating FIM, most notably K-FAC \citep{Martens2015} and Shampoo \citep{Gupta2018}. More recently, the SOAP optimizer \citep{Vyas2025} has improved upon Shampoo by rotating gradients into the principal curvature eigenspace of the Shampoo pre-conditioner and then applying RMSProp's diagonal scaling, before rotating back. All of these efficient approximations of Natural Gradient Descent can be realized in the DLS framework so long as one performs the necessary index shifting so that $V_g$ is independent of the local data at $\phi_g$.

Another gold standard is Newton's method \citep{Nocedal2006}, where one's preconditioner is proportional to the inverse Hessian matrix. Once again, however, computing and storing a full $N \times N$ matrix is often computationally intractable. Various types of quasi-Newton methods have been developed to efficiently approximate the inverse Hessian. For instance, the Generalized Gauss-Newton method \citep{Schraudolph2002} proceeds via first-order Jacobian matrices, $H\approx J^T J$, whereas BFGS (and its memory-limited version, L-BFGS) approximate curvature using only recent first-order information \citep{Liu1989}. All of these quasi-Newton methods can be realized in the DLS framework (after the necessary index shifting).

Beyond approximating these two gold standards, emerging methods such as the Muon optimizer \citep{Jordan2024} are carving a new path entirely by using Newton-Schulz iterations to orthogonally precondition the update, preventing the gradient from losing rank over time. See also its recent extension: Muon2 \citep{Muon2} and Newton-Muon \citep{NewtonMuon}. Given that the pre-conditioners used in all of the above-discussed methods may change over time, Eq.~\eqref{EqDLSNoiseRel} describes exactly what modulation of the shape and size of genetic drift is required to maintain biological fidelity.

Viewed through the lens of the DLS framework, all of these advanced optimizers are nothing more than specific analytical bookkeeping choices for the sub-population variance, $V_g$. Importantly, choosing to work with one type of pre-conditioner or another does not amount to postulating that the population at large has a corresponding variance structure. For instance, working with a diagonal pre-conditioner (as in RMSProp) does not amount to an assumption that something prevents the species from having covariances between its genes. Instead, it merely means choosing to analyze the wider population (whatever its variance structure) in terms of sub-populations which lack such covariances. In general, there is no reason to assume that the variance of the lineage in question will naturally evolve towards one's preferred kind of pre-conditioner, be it an inverse FIM or some other Kronecker-factored geometry. Rather, at every generation, we are manually dividing the lineage into sub-populations with a certain anisotropic shape and then picking one at random to continue following. Concretely, we might pick for their variances, $V_g$, to match the matrices generated by K-FAC, Shampoo, SOAP, or L-BFGS. By dynamically shaping the sub-populations in this way, the resulting lineages may more efficiently navigate the fitness landscape. 

At this point, one may feel the need to cry ``Foul!'' that we are somehow picking winners and losers and thereby unnaturally speeding evolution along. But one must recall that these sub-populations do not carry any special demographic ontology; it is purely a matter of analytical bookkeeping. The key mechanism which keeps our bookkeeping honest is the DLS noise relation: concretely, if one chooses for the next batch of sub-populations to be slimmer along the x-axis, then the distribution of their means will be correspondingly wider. Therefore, when one randomly selects one of these sub-populations there will inevitably be an increased amount of genetic drift along the x-axis. Hence, our choice to reduce the amount of ``fuel'' which selection has in a certain direction is automatically balanced by an increased amount of noise in that direction. Alternatively, we might increase the amount of selection in a certain direction, balanced by decreased noise in that direction. Ultimately, by Theorem~\ref{ThmDLSUpdatePackage}, the total population's Fisherian dynamics, Eq.~\eqref{EqStochasticGenUpdateRule}, can be perfectly recovered by reassembling an ensemble of DLS simulations, \textit{no matter what bookkeeping choices we make.}

But one must not go too far in the other direction either; it would be a mistake to label our bookkeeping choices ``gauge'' in the physicist's sense and cry ``Unphysical!'' about every decomposition-dependent aspect of our evolutionary simulation. To the contrary, our sub-populations are real collections of individuals (e.g., grouped together because they share similar genotypes). Over time these sub-populations form evolutionarily-salient lineages (daughters of daughters of daughters, etc.). Suppose that after specifying some bookkeeping (``gauge'') choices we simulate exactly one lineage and notice that it develops adaptation $X$ in time $T$. This ostensibly ``gauge-dependent'' fact amounts to a real achievement by this lineage in particular as well as a real part of the total population's evolutionary history. (And, of course, more so if accomplished by multiple lineages.) Recall from Eq.~\eqref{EqLineageWeight} that the demographic weight of any lineage is the exponential of its accumulated log-fitness. Hence, by utilizing advanced optimizers to find lineages which quickly climb the fitness landscape, we are computationally isolating those exact lineages which most contribute to the total population's evolutionary history.

It is revealing to consider the worst possible way of running a DLS simulation while maintaining its evolutionary fidelity. Namely, one might begin from an individual genotype ($V_0=0$) and then down-sample to an individual genotype at every subsequent step ($V_g=0$). In this case, because the lineage never has any variance, there is never any ``fuel'' for selection to act upon. The only way that the lineage's genotype, $\phi_g$, moves in this case is through genetic drift. At every step, the amount that one must down-sample by to maintain $V_{g+1}=0$ is $W_g=\mu^2I$. Each child is just one random mutation, $\xi_g\sim\mathcal{N}(0,\mu^2I)$, away from their parent. For these bookkeeping choices, an ensemble of DLS simulations would simply amount to myriad unbiased random walks heading out in every direction. By random chance (again, no selection pressure), some of these lineages will wander into a higher fitness region. An even smaller subset of these lineages will (again, by random chance) stay near a fitness peak. As rare as these needles-in-a-haystack are, they will receive a tremendous demographic weight when it comes time to reassemble the total population, see Eq.~\eqref{EqLineageWeight}. As un-illuminating as this pure-drift rendition of evolution is, it is completely faithful to the total population's deterministic Fisherian dynamics. Indeed, it is a special case of the asexual Fisher-Wright equivalence.

The key difference between this Pure Drift view and Fisher's view is how much of the story is being accounted for by selection effects \textit{within lineages} vs \textit{between lineages}. 
Concretely, the Pure Drift view eschews all variance (and hence, selection) in the lineage dynamics, instead pushing all of the variance (and hence, selection) into the space between the lineages, as revealed at the reassembly phase. By contrast, one can think of Fisher's view as involving only one lineage (i.e., the total population) which alone carries all of the population's variance. The DLS framework is situated (alongside Wright) somewhere in the middle of this spectrum. It seeks to make each lineage's variance as large as possible (while maintaining computational efficiency) so that each individual lineage becomes more meaningful, capturing more of the evolutionary story. From this perspective, one can think of the above-discussed techniques as helping us locate the most demographically weighty lineages without having to grind through all simulating all of the trillions of failed lineages.

So far, our discussion in this subsection has operated under the assumption that the variance of the lineage that we are tracking is small relative to the local curvature of the fitness landscape ($V_g H_g \ll I$). If we relax this assumption we must account for the selection effect of the Hessian matrix itself, $H_g$. In doing so, we recover the Full DLS update rule, which maps perfectly onto second-order optimization---specifically, the Damped Newton's method.
\begin{corollary}[Damped Newton plus the DLS Noise Relation]\label{ThmDLSNoiseRelFull}
Consider a generic implementation of the Damped Newton's method with a variable inverse damping matrix, $V_g\succeq0$, and a variable Gaussian noise, $W_g\succeq0$, both being symmetric and positive semi-definite. Namely,
\begin{align}
&\phi_{g+1}
=\phi_g + (V_g^{-1}-H_g)^{-1}\,\nabla\log(\mathcal{F}_\lambda)(\phi_g) + \xi_g,
\qquad \xi_g\sim\mathcal{N}(0,W_g),
\end{align}
where $H_g$ is the Hessian of the objective function, $\log(\mathcal{F}_\lambda)$, at $\phi_g$. In order for this to be a biologically faithful simulation of evolution (specifically, a DLS) it must satisfy the DLS noise relation for some fixed mutation rate, $\mu^2\geq0$,
\begin{align}
W_g=\mu^2I-V_{g+1}+(V_g^{-1}-H_g)^{-1}.
\end{align}
Moreover, the objective function, $\log(\mathcal{F}_\lambda)$, needs to be approximately quadratic within a region of several standard deviations of $\phi_{g}$ (as measured by $V_g$) with $V_{g}^{1/2} \ H_g \ V_{g}^{1/2}\preceq I$. Lastly, if $g>0$ then $V_g$ must be independent of local data at $\phi_{g}$.
\end{corollary}
\begin{proof}
This is just a restatement of part of Theorem~\ref{ThmDLSUpdatePackage} and the above discussion.
\end{proof}
\noindent In many applications, utilizing the raw inverse Hessian as a pre-conditioner has been found to be notoriously unstable. Consequently, it is standard practice to introduce a damping matrix to the above-discussed Newton's method to ensure that the update step remains well-behaved \citep{Nocedal2006}. In particular, this damping mechanism is mathematically equivalent to establishing a ``region of trust,'' which strictly bounds the update step to prevent the algorithm from making wild leaps when the local quadratic model is only reliable locally. Remarkably, in the Full DLS update, the inverse biological variance of the lineage, $V_g^{-1}$, natively fulfills this exact mathematical role. The variance of the lineage automatically stabilizes the second-order optimization process. This will be a faithful simulation of evolution provided the genetic drift, $W_g$, is injected precisely according to the DLS noise relation.

Thus far, we have demonstrated that advanced optimization techniques like pre-conditioning and Hessian damping emerge organically from population genetics. However, not every optimization algorithm is evolutionarily compliant right out of the box. A notable outlier is the Adam optimizer, arguably the most ubiquitous optimization algorithm in modern deep learning. While Adam's diagonal pre-conditioner (i.e., RMSProp) poses no major issues, its formulation relies on a physics-inspired momentum term that gives the algorithm some ``inertia'' so that it can ``coast'' across flat landscapes \citep{Polyak1964}. In the next section, we will try to cast Adam into the DLS framework, exposing exactly why this inertial form of momentum violates evolutionary first principles. Fortunately, with only minimal surgery, Adam can be brought back into strict evolutionary compliance.

\subsection{Adam-DLS: Bringing Adam into Evolutionary Compliance}\label{SecRecoverAdam}
The Adam optimizer is arguably the most successful (or, at least, most ubiquitous) first-order optimization heuristic in modern machine learning. In this section, we recast Adam into the Darwinian Lineage Simulation (DLS) framework, isolating its non-evolutionary artifact (additive momentum), and demonstrate the minimal analytical surgery required to recover Adam's advantages purely through the variance shaping techniques described above.

Let the current fitness gradient be denoted as $f_g = \nabla \log \mathcal{F}_\lambda(\phi_g)$ with its stochastic dependence on $\lambda$ suppressed. Just as in our above discussion of RMSProp, Adam tracks an exponential moving average of the squared gradients for each parameter, $s_{g+1}=\beta_2 s_{g}+(1-\beta_2) f_g^2$. Unlike RMSProp, Adam also tracks an exponential moving average of the recent gradients, $m_{g+1} = \beta_1 m_g + (1 - \beta_1) f_g$ where $\beta_1\in[0,1)$ controls the decay rate. One can define $\hat{m}_g = m_g / (1 - \beta_1^{g})$ to correct for an initialization bias which results from taking $m_0=0$. It should be noted that in our indexing convention, the quantities $s_g$ and $m_g$ represent the second and first moments of the \textit{past gradients}, not yet including $f_g^2$ and $f_g$ respectively. To include these, the index must be advanced by one from $g$ to $g+1$. 

In this notation, the standard Adam update step plus a naive implementation of mutation-induced genetic drift, $\xi_g^\text{naive}\sim\mathcal{N}(0,\mu^2I)$, is then,
\begin{align}
\nonumber
&\text{Noisy Adam Update:},\\
\nonumber
&\text{Initialize: }\quad m_0=0,\quad s_0=0,\\
\nonumber
&\phi_{g+1}= \phi_g + \frac{\alpha}{\sqrt{\hat{s}_{g+1}}+\epsilon} \ \hat{m}_{g+1}+\xi_g^\text{naive}\\
\label{EqStandardAdam}
&\qquad=\phi_g + \frac{\alpha}{\sqrt{s_{g+1}/ (1 - \beta_2^{g+1})}+\epsilon} \frac{(1-\beta_1)f_g+ \beta_1 m_g}{1-\beta_1^{g+1}}+\xi_g^\text{naive}
\end{align}
for some base learning rate $\alpha$ with some small $\epsilon$ included for numerical stability. The $\alpha/(\sqrt{\hat{s}_{g+1}}+\epsilon)$ term acts as a diagonal preconditioner (just as in RMSProp) which now acts on the current gradient, $f_g$, as well as the historical gradients, $\hat{m}_{g}$. Recall from Corollary~\ref{ThmDLSNoiseRelLeading}, however, that it is forbidden for the pre-conditioner to depend upon $f_g$, which it does here through $s_{g+1}$. The fix now is as simple as it was in the RMSProp case: we can use $\sqrt{s_g/ (1 - \beta_2^{g+1})}$ instead. This, however, means that we cannot initialize with $s_0=0$ lest we hit a divergence at generation $g=0$. Instead, $s_0$ will have to be initialized so as to yield the user's specification for, $V_0$, the lineage's initial (diagonal) variance. As was noted above, a very natural choice would be $s_0=(1-\beta_2)f_0^2$ from which it follows that $s_0/(1-\beta_2)=s_1/(1-\beta_2^2)=f_0^2$. This choice keeps the pre-conditioner perfectly stable through the early updates while every subsequent $s_g$ mechanically mimics Adam's usual scheme.

The first phase of our surgery on Adam is now complete. To recap our progress: We have shifted the index on the second moments back by one and introduced a non-trivial initialization. Concretely, we have:
\begin{align}
\nonumber
&\text{Mid-Surgery Noisy Adam Update:},\\
\nonumber
&\text{Initialize: }\quad m_0=0,\quad s_0=(1-\beta_2)f_0^2,\\
\nonumber
&\phi_{g+1}= \phi_g + \frac{\alpha}{\sqrt{s_{g}/ (1 - \beta_2^{g+1})}+\epsilon} \frac{(1-\beta_1)f_g+ \beta_1 m_g}{1-\beta_1^{g+1}}+\xi_g^\text{naive}\\
&\qquad= \phi_g + \underbrace{(1-\beta_1)D_g f_g}_{\text{Current Gradient Step}} + \underbrace{\beta_1 D_g m_{g}}_{\text{Past Gradient Step}}+\xi_g^\text{naive},
\end{align}
where $\xi_g^\text{naive}\sim\mathcal{N}(0,\mu^2I)$ is a naive implementation of genetic drift. Notice that we have collected together all of the diagonal pre-conditioning terms into the following matrix, 
\begin{align}
D_g=\frac{\alpha}{\sqrt{s_{g}/ (1 - \beta_2^{g+1})}+\epsilon} \ \frac{1}{1-\beta_1^{g+1}}.  
\end{align}
One part of this above update step is in the (pre-conditioned) direction of the current gradient, $f_g = \nabla \log \mathcal{F}_\lambda(\phi_g)$, and another part is in the (pre-conditioned) direction of the past gradients, $m_g$. This update step is still not of the leading order DLS form: $\phi_{g+1} = \phi_g + V_g f_g + \xi_g$. In any Darwinian framework, directional adaptation requires selection pressure. In the absence of a fitness gradient, $f_g=0$, evolution can only move through the landscape via random genetic drift, $\xi_g$. By contrast, Adam has an inertial form of momentum, continuing in the direction of its summarized history, $D_g m_g$.

We can bring Adam into compliance with evolutionary dynamics, however, by ensuring that its historical momentum effects only act as an amplifier for existing selection pressures. The simplest way to do this is to replace $m_g$ with a projector from $f_g$ onto $m_g/\|m_g\|$. In order to keep the pre-conditioner symmetric (so that it can be interpreted as a population variance) it is then necessary to place $D_g$ symmetrically on either side of the projector. The resulting pre-conditioner for Adam-DLS is,
\begin{align}\label{EqAdamDLSVg}
V_g \coloneqq (1-\beta_1)D_g + \beta_1 \ 
\begin{cases}
D_g \frac{m_g m_g^T}{m_g^T D_g m_g} D_g& , \ m_g^T D_g m_g\neq0\\
0 & , \ m_g^T D_g m_g=0
\end{cases}.
\end{align}
This is a diagonal pre-conditioner plus a rank-1 matrix extending the lineage's variance in the $D_g m_g$ direction. This will cause increased selection pressure along the direction of historically successful adaptations. It will be useful momentarily to have also defined the following scalar projection term,
\begin{equation}
d_g \coloneqq 
\begin{cases}
\frac{m_g^T D_g f_g}{m_g^T D_g m_g}& , \ m_g^T D_g m_g\neq0\\
1 & , \ m_g^T D_g m_g=0
\end{cases}.
\end{equation}

Plugging this symmetric covariance matrix into the leading order DLS update rule yields:
\begin{align}
\nonumber
&\text{Adam-DLS Update:}\\
\nonumber
&\text{Initialize: }\quad m_0=0,\quad s_0=(1-\beta_2)f_0^2,\\
\nonumber
&\phi_{g+1}^{\text{}} = \phi_g + (1-\beta_1)D_g f_g + \beta_1 \ d_g \ D_g \ m_g+\xi_g\\
&\qquad= \phi_g + \frac{\alpha}{\sqrt{s_g/ (1 - \beta_2^{g+1})}+\epsilon} \ \frac{(1-\beta_1)f_g+ \beta_1 d_g m_g}{1-\beta_1^{g+1}} + \xi_g
\end{align}
where $\xi_g\sim\mathcal{N}(0,W_g)$ is no longer naive but now fixed by the DLS Noise relation. This is exactly like the standard Noisy Adam update rule except with three changes: 
\begin{itemize}
    \item[1.] We have shifted the index on the second moments back by one (replacing $\sqrt{s_{g+1}/ (1 - \beta_2^{g+1})}$ with $\sqrt{s_g/ (1 - \beta_2^{g+1})}$). Notice that the same $(1-\beta^x)$ factor is applied as in vanilla Adam at each step. We also introduced a non-trivial initialization, $s_0 = (1-\beta_2)f_0^2$.
    \item[2.] The noise term, $\xi_g$, is now structured by the DLS noise relation which allows Adam-DLS to cross plateaus in an evolutionarily compliant way (i.e., random genetic drift). See Appendix~\ref{AppendixAdamDLSNoise} for a guide on how to efficiently draw samples from the DLS noise profile associated with Adam-DLS.
    \item[3.] Lastly, there is a new scalar, $d_g$, appearing in front of the historical momentum term, $D_g m_g$. This scalar modifies the size of the momentum term depending on how aligned the current gradient, $f_g$, is with the pre-conditioned past gradients, $D_g m_g$. If $f_g=m_g$ is aligned with the historical gradients then $d_g=1$ and we perfectly recover Adam's momentum term. By contrast, if $f_g$ is orthogonal to $m_g$ (understood with respect to $D_g$) then the rank-1 addition to the diagonal variance is not felt at the selection step and there is no momentum effect. Beyond these two cases, the part of $f_g$ which aligns with $m_g\ m_g^T$ is always amplified.
\end{itemize}
See Fig.~\ref{FigRosenbrock} for a direct comparison of Adam and Adam-DLS on the Rosenbrock test function. Both pass this benchmark demonstrating that our surgery has not incapacitated Adam's ability to navigate ill-conditioned landscapes.

\begin{figure}[t!]
\begin{equation*}
\begin{aligned}
    &\textbf{Adam-DLS Algorithm}\\
    &\rule{130mm}{0.4pt} \\
    &\textbf{input}      : \alpha \text{ (lr)}, \beta_1, \beta_2 \text{ (betas)}, \delta \text{ (noise safety)}, \epsilon \text{ (var. safety)}, \mu^2 \text{ (mutation rate)}\\
    &\phi_0 \text{ (init. genotype)}, \log\mathcal{F}_\lambda(\phi) \text{ (log fitness)},  \\
    &\textbf{initialize} :  m_{0} \leftarrow 0 \text{ (first moment)}\\[-1.ex]
    &\rule{130mm}{0.4pt} \\
    &f_0 \leftarrow \nabla \log\mathcal{F}_\lambda(\phi_0)\\
    &s_{0} \leftarrow (1-\beta_2) f_0^2\\
    &\textbf{for} \: g=0 \text{ to } \dots \textbf{do} \\
    &\hspace{5mm} f_g \leftarrow \nabla \log\mathcal{F}_\lambda(\phi_g) \\
    &\hspace{5mm} m_{g+1} \leftarrow \beta_1 m_g + (1 - \beta_1) f_g \\
    &\hspace{5mm} s_{g+1} \leftarrow \beta_2 s_g + (1 - \beta_2) f_g^2 \\
    &\hspace{5mm} D_g\leftarrow\frac{\alpha}{\sqrt{s_g/(1-\beta_2^{g+1})}+\epsilon} \ \frac{1}{1-\beta_1^{g+1}}\\
    &\hspace{5mm} D_{g+1}\leftarrow\frac{\alpha}{\sqrt{s_{g+1}/(1-\beta_2^{g+2})}+\epsilon} \ \frac{1}{1-\beta_1^{g+2}}\\    
    &\hspace{5mm} \textbf{if } m_g^T D_g m_g = 0 \textbf{ then } d_g \leftarrow 1 \\
    &\hspace{5mm} \textbf{else } d_g \leftarrow\frac{m_g^T D_g f_g}{m_g^T D_g m_g}\\
    &\hspace{5mm} \xi_g \leftarrow \text{AdamDLSNoise}(\alpha,\beta_1,\delta,\epsilon,\mu^2,m_g,D_g,m_{g+1},D_{g+1}) \qquad \text{(See Appendix~\ref{AppendixAdamDLSNoise})}\\
    &\hspace{5mm} \phi_{g+1} \leftarrow \phi_g + D_g \left((1-\beta_1)f_g+ \beta_1 d_g m_g\right) + \xi_g \\
    &\rule{130mm}{0.4pt} \\[-1.ex]
    &\textbf{return} \: \phi_G \\[-1.ex]
    &\rule{130mm}{0.4pt} \\[-1.ex]
\end{aligned}
\end{equation*}
\vspace{-1.0cm}
\end{figure}

\section{Conclusion}
This paper has performed a service of scientific hygiene that drastically improves the quality and quantity of gradient-based optimization algorithms which have been certified as evolutionarily faithful. A wide family of engineering tools have been transformed into scientifically valid simulations of Darwinian evolution in silico. In many cases, all that was needed was the DLS noise relation; in the case of Adam, light surgery. The \textit{Darwinian Lineage Simulation} (DLS) framework developed in this paper earns its evolutionary credentials by grounding itself explicitly in \citeauthor{Fisher1930}'s (\citeyear{Fisher1930}) and \citeauthor{Wright1931}'s (\citeyear{Wright1931,Wright1932}) bitterly opposed views of evolutionary dynamics. Remarkably, we have shown that these rival views become formally equivalent when adapted to an asexual context.

Both of these achievements share a common engine: a proper understanding of genetic drift. Prior treatments either replaced it with thermal noise \citep{Whitelam2021}, left it underspecified \citep{Frank2025}, or restricted themselves to isotropic approximations \citep{Kucharavy2023}. By contrast, the DLS framework derives its model of genetic drift from evolutionary first principles. Ultimately, genetic drift comes from sampling smaller populations or individuals out of a larger distribution. By Fisher's key insight, Eq.~\eqref{EqFisherKeyInsight}, the size and shape of these distributions acts as a pre-conditioner, $V_g$, on the fitness gradient. Hence, any proper model of genetic drift must intimately connect its shape and size to how these pre-conditioners are changing from one generation to another. This line of thought leads directly to the DLS noise relation and thereby opens the door to every result in this paper.

Regarding the asexual Fisher-Wright equivalence, another structural move underpins this result. Restricting to asexual reproduction transforms Wright's sub-populations from an ecological commitment into a pure bookkeeping device. The DLS noise relation then emerges as the unique constraint ensuring that one's bookkeeping is honest: an ensemble of such lineages must faithfully recover the full population dynamics. Beyond this, however, \textit{any} valid variance schedule $\{V_g\}$ constitutes a legitimate DLS.

This bookkeeping freedom then underwrites this paper's principal technical result: the retroactive certification of evolutionary fidelity for a broad family of advanced optimization algorithms. Stochastic Gradient Descent as well as many regularizations/approximations of Newton's method and Natural Gradient Descent are all revealed to be evolutionarily faithful DLS implementations once coupled with the DLS noise relation. A notable outlier is Adam, whose inertial momentum term permits deterministic movement in the absence of selection pressure --- an evolutionary impossibility. The principled repair performed in Sec.~\ref{SecRecoverAdam} replaces its additive momentum with a rank-1 projection that amplifies selection pressure along historically successful directions. This recovers full momentum when current and historical gradients broadly align and gracefully eliminates it when they are orthogonal. As the Rosenbrock benchmark (Fig.~\ref{FigRosenbrock}) confirms, this surgery preserves Adam's capacity to navigate ill-conditioned landscapes.

This capacity matters not merely as a formal result. Consider the hypothesis that organisms evolve effective unconscious physics simulators --- a conjecture that has motivated Baldwinian neuroevolution of Physics-Informed Neural Networks (PINNs) \citep{Wong2026}. Because PINNs embed partial differential equations directly into the loss function, the resulting loss landscape is notoriously ill-conditioned and riddled with degenerate saddle points due to the high-order differential operators in the PDE residual terms \citep{Rathore2024, DeRyck2024,Kiyani2026}. For the evolutionary biologist, this represents a fitness landscape with extreme epistasis. First-order isotropic methods (like basic SGA) or even diagonal methods (like RMSProp) struggle profoundly in such stiff landscapes because the simulated dynamics cannot reliably follow the sharp fitness ridges (recall Fig.~\ref{FigRosenbrock}). However, utilizing quasi-Newton pre-conditioning (e.g., L-BFGS) has been empirically proven to reduce the magnitude of the eigenvalues and the condition number of the PINN Hessian by a factor of $1000$ or more, drastically accelerating convergence \citep{Rathore2024}. By defining $V_g$ via L-BFGS or quasi-Newton pre-conditioning and coupling it with the DLS noise relation, an evolutionary simulation is now available that brings exactly the anisotropic variance these ill-conditioned landscapes demand while preserving evolutionary fidelity. 

Here is a concrete preview of what these new DLS-certified algorithms might enable. Consider how fitness-enhancing our ability to unconsciously simulate physics is, and postulate an analogous fitness score which quantifies a neural network's ability to do so (or to quickly learn to do so).\footnote{As with all scientific analogues, this posit will live or die by its fruitfulness. Are relatively simple fitness models capable of accounting for the evolution of core innate knowledge? Who knows. But we can guess, and check, and modify/complicate and debate over our guesses until we have a working account with the requisite level of complexity.} Running a DLS simulation on this fitness landscape is then an analogical model of how we ourselves might have come to have this ability (in the distant past, long before we were even mammals). This analogical sandbox allows us to ask: How exactly did this \textit{core knowledge} develop in us? What kinds of fitness environments lead to a shallow overfitting vs a deep innate structural knowledge of physics? How is this innate structure realized in neurons? It is important to separate the scientific questions (which require scientific hygiene and careful analogical reasoning) from what the engineer might ask: How can we most effectively replicate this kind of core knowledge in AI?

Good scientific instruments require purposeful construction, especially when reasoning by analogy. By grounding the DLS framework in population genetics rather than statistical physics, every certified algorithm constitutes a genuine in silico evolutionary simulation --- one whose parameters carry biological meaning and whose limits can be systematically interrogated. Looking ahead, the DLS perspective is naturally related to other optimization methods that maintain and update Gaussian distributions over a genotype space. These include: Covariance Matrix Adaptation (CMA-ES), Natural Evolution Strategies (NES), and Estimation of Distribution Algorithms (EDAs). The non-trivial variance structure of these gradient-free methods allows them to also pass the Rosenbrock benchmark \citep{RosenbrockCMA-ES}. Establishing the precise biological interpretation of their covariance dynamics under the DLS framework is a productive direction for future work --- and a logical next step in fulfilling evolutionary computation's long-promised scientific mandate.

\section*{Acknowledgments} I am deeply grateful to my brother, Benjamin Grimmer, for teaching me so much optimization theory over the years. I also wish to thank my student, Orion Shtrezi, for their sharp feedback and insightful questions which helped refine this work.

\small

\bibliographystyle{apalike}
\bibliography{BibDirectFromDarwin}

\appendix
\section{Tractable Sampling of Genetic Drift for Adam-DLS}\label{AppendixAdamDLSNoise}
This appendix will work through some subtleties which arise when applying the DLS noise relation to Adam-DLS in order to achieve an evolutionarily faithful model of genetic drift. The algorithm described here is available at \url{https://github.com/danielgrimmer/adam-dls}. Our ultimate goal is to draw a sample, $\xi_g\sim \mathcal{N}(0, W_g)$, where $W_g = \mu^2 I - (V_{g+1} - V_g)$ and where $V_g$ is defined by Eq.~\eqref{EqAdamDLSVg}. To do so, we shall first draw from an isotropic distribution, $z \sim \mathcal{N}(0, I)$, from which our desired noise profile follows as $\xi_g=W_g^{1/2}z$. Hence, our task reduces to computing $W_g^{1/2}$ and ensuring that it is positive semi-definite. Note that our desired $W_g^{1/2}$ does not have to be the principal square root of $W_g$, it could be any valid sampling factor (such as a Cholesky factor) which ``squares'' properly as $W_g^{1/2}(W_g^{1/2})^T=W_g$. 

Importantly, given some user-specified or adaptive sequence of variances, $\{V_g\}$, it is not guaranteed that $W_g$ will be positive semi-definite. Indeed, recall from Sec.~\ref{SecDLSNoiseRel} that evolution naturally imposes a ``speed limit'' on how fast the lineage's variance can grow or change; the mutation rate $\mu^2$ might not be large enough to allow for $V_g$ to suddenly become larger, or to re-orient its direction of elongation, e.g., from the x-axis to the y-axis. In order to meet the user's demands the mutation rate $\mu^2$ would need to be higher. This motivates the following approach to handling such errors. The only way to correctly down-sample from $V_g+\mu^2I$ to the target variance $V_{g+1}$ while maintaining evolutionary fidelity is to introduce a \textit{one-off ad hoc mutation spike} to ensure $W_g\succeq0$. Along with this mutation spike we shall throw a \textit{soft error} to the user telling them that their next run of the simulation needs to have a higher mutation rate. As a first pass, however, let us assume that $\mu^2$ is large enough such that no errors occur. After we see how $W_g^{1/2}$ is computed in this case, we can then handle the potential error cases. 

To begin, let us define the scaled historical momentum vector:
\begin{equation}
y_g = \begin{cases}
\sqrt{\beta_1} \frac{D_g m_g}{\sqrt{m_g^T D_g m_g}}& , \ m_g^T D_g m_g\neq0\\
0 & , \ m_g^T D_g m_g=0
\end{cases}.
\end{equation}
By substituting the Adam-DLS variance, Eq.~\eqref{EqAdamDLSVg}, into the leading order DLS noise relation one finds, 
\begin{equation}
W_g = \underbrace{\left[ \mu^2 I - (1-\beta_1)(D_{g+1} - D_g) \right]}_{\text{Diagonal Base: } S_g} + \underbrace{y_g y_g^T}_{\text{Positive Rank-1}} - \underbrace{y_{g+1} y_{g+1}^T}_{\text{Negative Rank-1}}.
\end{equation}
This is a diagonal matrix, $S_g$, modified by a low-rank (rank-2) update with mixed signs, plus $y_g y_g^T$ minus $y_{g+1} y_{g+1}^T$. Fortunately, this form allows for us to compute $W_g^{1/2}$ efficiently as follows. We begin by factoring out the diagonal base $S_g$ from the rank-2 part of $W_g$. Concretely, let $\tilde{y}_g = S_g^{-1/2} y_g$ and $\tilde{y}_{g+1} = S_g^{-1/2} y_{g+1}$. Stack these into an $N \times 2$ matrix $U = [\tilde{y}_g, \tilde{y}_{g+1}]$. Defining a $2 \times 2$ signature matrix $C = \text{diag}(1, -1)$, we can rewrite the noise covariance as:
\begin{equation}
W_g = S_g^{1/2} (I + U C U^T) S_g^{1/2}.
\end{equation}
Hence, our task reduces to finding a square root of the inner matrix $(I + U C U^T)$.

To do so, we perform a thin QR decomposition, $U = Q R$, where $Q$ is an $N \times 2$ matrix with orthonormal columns ($Q^T Q = I$) and $R$ is a $2 \times 2$ upper triangular matrix. The inner term becomes $I + Q(R C R^T)Q^T$. Defining a $2 \times 2$ symmetric matrix, $M = R C R^T$, we then have:
\begin{equation}
W_g = S_g^{1/2} (I + Q M Q^T) S_g^{1/2}.
\end{equation}
Next, we can perform a computationally trivial eigendecomposition $M = E \Theta E^T$ where $\Theta$ is diagonal. Letting $A = Q E$ we can substitute this back into the inner term to find,
\begin{equation}
W_g = S_g^{1/2} (I + A \Theta A^T) S_g^{1/2}.
\end{equation}
Note that because $Q$ has orthonormal columns and $E$ is orthogonal, $A$ also has orthonormal columns ($A^T A = I$). 

We are now in a position to compute $W_g^{1/2}$. We seek a matrix $X$ such that $XX^T = I + A \Theta A^T$. We propose a solution of the form $X = I + A K A^T$ where $K$ is a $2 \times 2$ diagonal matrix. It follows from $A^T A = I$ that $K$ needs to satisfy $2K + K^2 = \Theta$. This yields $K = \sqrt{I + \Theta} - I$ such that,
\begin{equation}
W_g^{1/2} = S_g^{1/2} \left(I + A \left(\sqrt{I + \Theta} - I\right) A^T\right) .
\end{equation}
Applying this transformation to some $z \sim \mathcal{N}(0, I)$ as $\xi_g = W_g^{1/2} z$ yields $\xi_g \sim \mathcal{N}(0, W_g)$ as desired.

There are two potential failure modes in the above derivation: either the diagonal base, $S_g$, or the core matrix $I + \Theta$ might fail to be positive semi-definite. Either failure forbids the computation of the required matrix square roots. Fortunately, because these conditions are jointly equivalent to the requirement that $W_g \succeq 0$, we can pre-emptively calculate a single, unified bound before $S_g$ or the $\tilde{y}$ variables are ever computed. If the baseline mutation rate is too low, we introduce a one-off ad hoc mutation spike to ensure evolutionary fidelity. Recall that $W_g = \mu^2 I + (1-\beta_1)(D_{g} - D_{g+1}) + y_g y_g^T - y_{g+1} y_{g+1}^T$. By Weyl's inequality, the minimum eigenvalue of $W_g$ is bounded below by the sum of the minimum eigenvalues of its components:
\begin{align}
\lambda_{\text{min}}(W_g) \ge \mu^2 + \underbrace{\lambda_{\text{min}}((1-\beta_1)(D_g - D_{g+1}))}_{\lambda_{\text{min}}^D} +     \underbrace{\lambda_{\text{min}}(y_g y_g^T - y_{g+1} y_{g+1}^T)}_{\lambda_{\text{min}}^Y}    
\end{align}
Since both $D_g$ and $D_{g+1}$ are diagonal, $\lambda_{\text{min}}^D$ is trivial to compute. We can calculate the minimum eigenvalue of the rank-2 component exactly. Using the identity $p p^T - q q^T = \frac{1}{2}(u v^T + v u^T)$ where $u = y_g + y_{g+1}$ and $v = y_{g} - y_{g+1}$, we find that the non-zero eigenvalues are exactly $\frac{1}{2}(u^T v \pm \|u\|\|v\|)$. Since $u^T v = \|y_g\|^2 - \|y_{g+1}\|^2$, the minimum eigenvalue is:
\begin{align}
\lambda_{\text{min}}^Y = \frac{1}{2}\left( \|y_g\|^2 - \|y_{g+1}\|^2 - \|y_g + y_{g+1}\| \|y_{g+1} - y_g\| \right)    
\end{align}
Therefore, to guarantee $W_g \succeq \delta I$ for some small safety factor $\delta > 0$, it is strictly sufficient that the mutation rate satisfies:
\begin{align}
\mu^2 \ge \mu^2_\text{req}=\delta-\lambda_{\text{min}}^D - \lambda_{\text{min}}^Y    
\end{align}
If the required mutation rate, $\mu^2_\text{req}$, ever exceeds the baseline mutation rate, $\mu^2$, then a one-shot ad hoc spike in the mutation rate will be applied to cover this difference. At any such occurrence, a soft error will be triggered alerting the user that they need to increase the baseline mutation rate to preserve evolutionary fidelity.

In general, $\mu^2$ must be large enough to simultaneously cover the maximum diagonal reduction in variance and the maximum anisotropic reorientation of the lineage. Crucially, because $m_g$ tracks an exponential moving average, the term $y_{g+1} - y_g$ is typically very small in practice. The $D_{g+1} - D_g$ term should also be relatively small as it too is derived from a moving average. The baseline mutation rate $\mu^2$ should be initialized large enough to encompass this dynamic bound. The above discussion has established a mathematically rigorous link between standard machine learning hyperparameters and the natural speed limits of Darwinian evolution.

\begin{figure}[t!]
\begin{equation*}
\begin{aligned}
    &\textbf{Adam-DLS Noise Sampling}\\
    &\rule{130mm}{0.4pt} \\
    &\textbf{input}      : \alpha \text{ (lr)}, \beta_1,\beta_2\text{ (betas)}, \delta \text{ (noise safety)}, \epsilon \text{ (var. safety)}, \mu^2 \text{ (mutation rate)}\\
    &f_g \text{ (current grad.)}, m_g, m_{g+1} \text{ (first moments)}, D_g, D_{g+1} \text{ (diag. pre-conditioners)}\\[-1.ex]
    &\rule{130mm}{0.4pt} \\
    &\textbf{if } m_g^T D_g m_g = 0 \textbf{ then } y_g \leftarrow 0 \\
    &\textbf{else } y_g\leftarrow\sqrt{\beta_1} \frac{D_g m_g}{\sqrt{m_g^T D_g m_g}}\\    
    &\textbf{if } m_{g+1}^T D_{g+1} m_{g+1} = 0 \textbf{ then } y_{g+1} \leftarrow 0 \\    
    &\textbf{else } y_{g+1}\leftarrow\sqrt{\beta_1} \frac{D_{g+1} m_{g+1}}{\sqrt{m_{g+1}^T D_{g+1} m_{g+1}}}\\
    &\lambda_{\text{min}}^D \leftarrow \text{min}((1-\beta_1)(D_{g} - D_{g+1}))\\
    &\lambda_{\text{min}}^Y \leftarrow \frac{1}{2}\left( \|y_g\|^2 - \|y_{g+1}\|^2 - \|y_g + y_{g+1}\| \|y_{g+1} - y_g\| \right)\\
    &\mu^2_\text{req} \leftarrow \delta-\lambda_{\text{min}}^D - \lambda_{\text{min}}^Y \\
    &\textbf{if } \mu^2 < \mu^2_\text{req} \textbf{ then}\text{ Alert User of Soft Error:} \\
    &\hspace{7mm} \mu^2_\text{spike} \leftarrow \mu^2_\text{req}\\
    &\textbf{else}\\
    &\hspace{7mm} \mu^2_\text{spike} \leftarrow \mu^2\\
    &S_g\leftarrow\mu_\text{spike}^2 I - (1-\beta_1)(D_{g+1} - D_g)\\
    &\tilde{y}_g\leftarrow S_g^{-1/2}y_{g}\\
    &\tilde{y}_{g+1}\leftarrow S_g^{-1/2}y_{g+1}\\
    &U\leftarrow[\tilde{y}_{g},\tilde{y}_{g+1}]\\
    &Q,R\leftarrow \text{DecompQR}(U)\\
    &C\leftarrow\text{diag}(1,-1)\\
    &M\leftarrow R C R^T\\
    &\Theta,E\leftarrow \text{DecompEigen}(M)\\
    &A\leftarrow Q E\\
    &K\leftarrow\sqrt{I+\Theta}-I\\
    &z\leftarrow\text{Normal}(0,I)\\
    &z_1\leftarrow A K A^T z\\
    &\xi_g\leftarrow S_g^{1/2}(z+z_1)\\
    &\rule{130mm}{0.4pt} \\[-1.ex]
    &\textbf{return} \: \xi_g \\[-1.ex]
    &\rule{130mm}{0.4pt} \\[-1.ex]
\end{aligned}
\end{equation*}
\end{figure}

\end{document}